%% file: neurips2026.tex
\documentclass{article}

\PassOptionsToPackage{numbers,compress}{natbib}
\usepackage[preprint]{meta/neurips_2026}

\usepackage[utf8]{inputenc}
\usepackage[T1]{fontenc}
\usepackage{hyperref}
\usepackage{url}
\usepackage{booktabs}
\usepackage{amsfonts}
\usepackage{nicefrac}
\usepackage{microtype}
\usepackage{xcolor}
\hypersetup{
  pdftitle={Conservation Laws for Diffusion Models},
  pdfauthor={Ziv Aharoni and Henry D. Pfister},
  hidelinks
}

\usepackage{amsmath,amssymb,amsthm,mathtools}
\usepackage{graphicx}
\usepackage{wrapfig}
\usepackage{placeins}
\usepackage{tikz}
\usepackage{pgfplots}
\newif\ifarxiv

\arxivtrue
\ifarxiv
\makeatletter
\renewcommand{\@notice}{}
\makeatother
\fi
\graphicspath{{./}}

\pgfplotsset{compat=1.18}
\pgfplotsset{table/search path={./}}
\usepgfplotslibrary{groupplots,fillbetween}

\theoremstyle{plain}
\newtheorem{theorem}{Theorem}[section]
\newtheorem{lemma}[theorem]{Lemma}

\newtheorem{proposition}[theorem]{Proposition}
\theoremstyle{definition}
\newtheorem{definition}[theorem]{Definition}

\newcommand{\E}{\mathbb{E}}

\newcommand{\CE}{\mathrm{CE}}
\renewcommand{\triangleq}{\coloneqq}

\title{Conservation Laws for Diffusion Models}
\author{
  Ziv Aharoni, \quad Henry D. Pfister \\
  Department of Electrical and Computer Engineering \\
  Duke University \\
  Durham, NC, USA \\
  \texttt{\{ziv.aharoni,henry.pfister\}@duke.edu}
}

\begin{document}

\maketitle
\begin{abstract}
  While autoregressive models optimize the exact data likelihood via the chain rule, diffusion models are typically trained with denoising objectives.
  We develop conservation laws based on generalized extrinsic information transfer (GEXIT) functions for a broad class of memoryless noise processes, showing that the data--model cross-entropy (CE) can be characterized \emph{exactly} as an integral of \emph{local} information-theoretic derivatives along the noise path.
  This yields a unified characterization of the likelihood for discrete and continuous diffusion, with the Gaussian case reducing to the well-known mutual information--minimum mean-square error (I-MMSE) relationship.
  An immediate implication is a \emph{locality} property: one can compute the information-theoretic derivatives using only the marginal posteriors along the noise path.
  As a result, training reduces to learning the marginal posteriors by minimizing the negative log-likelihood.
  While the conservation law implies that the entropy does not depend on the noise path, finite-capacity denoisers approximate the posteriors with varying accuracy across noise types, leading to differences in performance.
  We validate these predictions on synthetic Markov sources and standard benchmarks, including \texttt{text8} and CIFAR-10.
  \ifarxiv
  Code is available at \url{https://github.com/zivaharoni/conservation-laws-diffusion-models}.
  \fi
\end{abstract}

\input{sections/intro}
\input{sections/gexit_background}
\input{sections/mismatched_exit}

\input{sections/implementation}
\input{sections/experiments}
\input{sections/limitations}
\input{sections/related_work}
\input{sections/conclusion}

\bibliographystyle{unsrtnat}
\bibliography{ref}

\appendix

\section{Deferred Proofs}\label{app:proofs}
\input{sections/appendix_proofs}

\section{Additional Experimental Details}\label{app:repro}
\input{sections/appendix_repro}

\ifarxiv
\else
\input{meta/checklist}
\fi
\end{document}

%% file: sections/intro.tex
\section{Introduction}
Autoregressive (AR) models treat the data distribution $P_X$, for $X = (X_1,\ldots,X_n)$, by factorizing it into a product of conditionals $P_X = \prod_{i=1}^n P_{X_i \mid X_{<i}}$. Thus, minimizing the negative log-likelihood (NLL) of an AR model $P^\theta_X = \prod_{i=1}^n P^\theta_{X_i \mid X_{<i}}$ is equivalent to minimizing the cross-entropy between $P_X$ and $P^\theta_X$. This enables exact likelihood evaluation but requires sequential generation, which can amplify distribution\nobreakdash -shift errors during long rollouts \cite{Oord2016Wavenet,Bengio2015ScheduledSampling}.
Diffusion models offer a complementary factorization through a forward noising process where $Y(\tau) = (Y_1(\tau),\dots,Y_n(\tau))$ is a noisy version of $X$ defined by $P_{Y(\tau) \mid X}$. The parameter $\tau\in [0,1]$ orders the channels by degradation, with $Y(0)$ corresponding to the original data and $Y(1)$ to pure noise.
The sampling (or reverse) process uses learned denoisers denoted $P^\theta_{X \mid Y(\tau)}$.

This step away from AR structure makes diffusion especially successful for image generation and increasingly attractive for text generation, including discrete diffusion, masked diffusion language models, and embedding-space text diffusion models \cite{SohlDickstein2015,Ho2020DDPM,Austin2021D3PM,Lou2023DiscreteDiffusionLM,Sahoo2024MDLM,Li2022DiffusionLM,Gulrajani2024LikelihoodDiffusionLM}.
Diffusion training typically optimizes local denoising objectives rather than exact sequence likelihood. Continuous denoising diffusion probabilistic model (DDPM)-style models are commonly motivated via variational bounds and equivalent weighted denoising losses \cite{SohlDickstein2015,Ho2020DDPM}, while in discrete denoising diffusion probabilistic models (D3PMs) and related discrete settings the relationship between denoising objectives and likelihood is usually indirect \cite{Austin2021D3PM,Lou2023DiscreteDiffusionLM,Sahoo2024MDLM}.

Recent work has begun to make this connection explicit by interpreting diffusion through information-theoretic derivatives along the noising path, including Gaussian results based on I-MMSE curves and discrete results for masked diffusion \cite{Kong2023InfoDiffusion,Jeon2025InfoDiscreteDiffusion,Bhattacharya2025ItDPDM}.
These analyses provide exact likelihood-scale characterizations for masked and Gaussian channels, but they do not yet yield a single formulation that applies across the broader class of memoryless noising mechanisms used in practice.
This motivates the following question.

\begin{center}
\fbox{\begin{minipage}{0.95\linewidth}
\centering
Can a single conservation law characterize the induced likelihood for diffusion models across memoryless noising processes, covering discrete and continuous diffusion?
\end{minipage}}
\end{center}
Our starting point for answering this question is a simple entropy telescoping identity. For any degraded path $0=\tau_0<\tau_1<\cdots<\tau_m=1$ from clean data to pure noise,
\begin{align*}
H(X)
&=H(X\mid Y(1))-H(X\mid Y(0)) =\sum_{j=0}^{m-1}\Big(H(X\mid Y(\tau_{j+1}))-H(X\mid Y(\tau_j))\Big),
\end{align*}
where each variable is corrupted independently with noise level $\tau$.
Each increment measures the information dissipation induced by adding a small amount of noise. If the mapping $\tau\mapsto H(X\mid Y(\tau))$ is absolutely continuous, and the spacing of the partition vanishes, then the telescoping sum converges to the integral of a derivative
\begin{align*}
H(X)&=\int_0^1 \frac{\partial}{\partial \tau} H(X\mid Y(\tau))\,d\tau.
\end{align*}
This is a conservation law: regardless of the noise choice, the integral is constant and equal to the entropy.
While it is conceptually appealing, it cannot be utilized without knowing the derivative of the joint conditional entropy $H(X\mid Y(\tau))$. However, for memoryless noise, the derivative simplifies to a sum of local contributions
\begin{align}\label{eq:gexit}
\frac{\partial}{\partial \tau} H(X\mid Y(\tau))
&=\sum_{i=1}^n \frac{\partial}{\partial t_i} H(X_i\mid Y_1(\tau),\dots,Y_i(t_i),\dots, Y_n(\tau))\big|_{t_i = \tau},
\end{align}
under mild conditions.
This gives the classical extrinsic information transfer (EXIT) / generalized EXIT (GEXIT) setting \cite{tenBrink2001EXIT,Ashikhmin-it04,Measson-it09}, which was developed originally to analyze the performance of iterative decoders for error-correcting codes.

The conservation-law view can also be applied to AR models: regardless of the order of conditioning the variables in the chain rule of probability, as in any-order AR models, the total entropy is conserved. In the same manner, regardless of the noise chosen to dissipate information along the noise path, the total entropy is conserved. Thus, diffusion models may be interpreted as a chain rule across the ``noise'' axis instead of the ``variable'' axis, and the counterpart of $P_{X_i|X<i}$ is $P_{X \mid Y(\tau)}$.

The discussion so far considers the matched setting, where the denoiser recovers the true posterior. In practice, diffusion models use learned posteriors, making the relevant quantity of interest a cross-entropy (CE) rather than entropy. In this work, we develop conservation laws for the CE in the mismatched setting.
This view has two immediate consequences.
First, the CE can be characterized using only marginal posteriors of the form $P_{X_i\mid Y(\tau)}$ for any admissible memoryless channel, rather than requiring the full joint posterior $P_{X\mid Y(\tau)}$.
Second, while the resulting conservation law does not depend on the noise type when the posteriors are exact, finite-capacity denoisers do not approximate all posterior families equally well, so different noise paths can induce substantial practical gaps.



We have conducted experiments designed to test this finite-capacity effect across synthetic Markov sources, text modeling, and image modeling. The synthetic experiments vary the alphabet size under a known entropy rate; the text experiments compare discrete and Gaussian noising paths on character-level language modeling; and the image experiments test whether the same channel-dependent finite-capacity effects appear for high-dimensional visual data.

Our contributions are fourfold:
\begin{itemize}
\setlength{\leftmargini}{0pt}
\item A unified likelihood analysis for diffusion: a mismatched GEXIT conservation law that exactly connects denoising to data--model cross-entropy across memoryless noise paths, including discrete and continuous diffusion.
\item Optimal denoising and training implications: marginal posteriors are sufficient to capture the data entropy and learning reduces to token-level negative-log-likelihood.
\item An implementation implication: it is sufficient to model $P_{X_i|Y_{-i}(\tau)}$ and tilt it with the channel.
\item Empirical validation: synthetic Markov sources, \texttt{text8}, and CIFAR-10 show large finite-capacity, channel-dependent gaps, making channels actionable design knobs.
\end{itemize}

%% file: sections/gexit_background.tex
\section{Background: GEXIT}
\subsection{Notation}
Throughout, $X=(X_1,\ldots,X_n)\in\mathcal{X}^n$ denotes a vector over a finite alphabet $\mathcal{X}$ with law $P_X$. 
We also assume that the output alphabet $\mathcal{Y}$ is either finite or equal to $\mathbb{R}^d$. For a noise-parameter vector $\mathbf{t}=(t_1,\ldots,t_n)\in\mathcal{T}^n$, where $\mathcal{T}\subseteq\mathbb{R}_{\ge 0}$ is the scalar noise-parameter domain and $\operatorname{int}(\mathcal{T})$ denotes its interior, we write
\begin{align*}
Y(\mathbf{t})&=(Y_1(t_1),\ldots,Y_n(t_n)) \in \mathcal{Y}^n
\end{align*}
for the coordinatewise corrupted observation, and we use $\mathbf{t}_{-i}$ and $Y_{-i}(\mathbf{t}_{-i})$ for the corresponding vectors with the $i$th coordinate removed. When the noise parameter is fixed and clear from context, we abbreviate $Y\equiv Y(\mathbf{t})$ and $Y_{-i}\equiv Y_{-i}(\mathbf{t}_{-i})$. A superscript $\theta$ denotes model-based quantities. We write $H(\cdot)$ for entropy and $\CE(\cdot,\cdot)$ for cross-entropy. For conditional CE terms such as $\CE(P_{X\mid Z},Q_{X\mid Z})$, the outer expectation is with respect to the first argument's marginal $P_Z$. Unless stated otherwise, expectations are taken with respect to the joint law induced by $P_X$ and the relevant channel family.

\subsection{GEXIT Background}
We follow the approach of \cite{Measson-it09} to define the class of memoryless channel families for which the GEXIT identities hold. 
We assume that all channels are either discrete or continuous with well-defined densities, and satisfy Convention (C1), defined at the start of Appendix~\ref{app:proofs}. This is the regularity condition used whenever we differentiate a channel family with respect to its noise parameter.
\begin{definition}[Admissible channel family]\label{def:admissible-channel-family}
Let $\mathcal{W}=\{W_t\}_{t\in\mathcal{T}}$ be a family of channels from $\mathcal{X}$ to $\mathcal{Y}$ satisfying Convention (C1).
We assume that $\mathcal{W}$ is degraded, which means that, for every $t'>t$, there exists a channel $Q_{t,t'}( y\mid \tilde y)$ on $\mathcal{Y}$ such that
\begin{align*}
W_{t'}( y\mid x)
&=\sum_{\tilde y\in\mathcal{Y}} W_t(\tilde y\mid x)\,Q_{t,t'}( y\mid \tilde y),
\end{align*}
where the sum is replaced by an integral for continuous $\mathcal{Y}$.
An admissible channel family is defined for every $\mathbf{t}=(t_1,\ldots,t_n)$ by 
\begin{align*}
P_{Y(\mathbf{t})\mid X}(y^n\mid x^n)
&=\prod_{i=1}^n W_{t_i}(y_i\mid x_i).
\end{align*}
\end{definition}

The basic GEXIT statement is that varying $t_i$ changes only the information carried by the $i$th coordinate. 
\begin{lemma}[GEXIT locality, adapted from \cite{Measson-it09}]
\label{lem:gexit-locality}
Let $\mathcal{W}=\{W_t\}_{t\in\mathcal{T}}$ be an admissible channel family. Then, for every input distribution $P_X$ and every noise vector $\mathbf{t}\in\operatorname{int}(\mathcal{T})^n$,
\begin{align}
\label{eq:gexit-background}
\hspace{-0.8em}\frac{\partial}{\partial t_i} H(X|Y(\mathbf{t}))
&=
\frac{\partial}{\partial t_i} H(X_i\mid Y(\mathbf{t}))  \!
\end{align}
\end{lemma}
The proof of Lemma~\ref{lem:gexit-locality} is given in Appendix~\ref{subsec:gexit-locality-proof}.
Lemma \ref{lem:gexit-locality} can be expressed in an alternative form in terms of the channel score function. Define the channel score function
\begin{align}
\label{eq:channel-score}
S_t(x,y)&\triangleq \frac{\partial}{\partial t}\log W_t(y\mid x).
\end{align}
The following identity rewrites GEXIT locality in score-function form and expresses the denoising posterior directly:
\begin{align}
\label{eq:gexit-score}
\frac{\partial}{\partial t_i} H(X\mid Y)
&= -\E\!\Big[
S_{t_i}(X_i,Y_i)\log P_{X_i\mid Y}(X_i\mid Y)
\Big].
\end{align}

This is the classical score-weighted GEXIT integrand \cite{Measson-it09,Measson-2006}, as proved in Appendix~\ref{subsec:gexit-score-proof}. 

\par Let $\gamma\colon(0,1)\to\operatorname{int}(\mathcal{T})^n$ be a continuously differentiable path in the noise-parameter space of $\mathbf{t}$, where $\mathcal{T}$ may be unbounded, with $\gamma(\tau)=(\gamma_1(\tau),\ldots,\gamma_n(\tau))$. We write $\gamma_i^\prime(\tau)$ for differentiation with respect to the path parameter $\tau$. Under Convention (C1), defined at the start of Appendix~\ref{app:proofs}, the map $\tau\mapsto H(X\mid Y(\gamma(\tau)))$ is absolutely continuous on compact subintervals of $(0,1)$. If the endpoint limits satisfy
\begin{align*}
\lim_{\tau\downarrow 0} H(X\mid Y(\gamma(\tau)))&=0, \;
\lim_{\tau\uparrow 1} H(X\mid Y(\gamma(\tau)))=H(X),
\end{align*}
we have the area theorem from the fundamental theorem of calculus:
\begin{align}
\label{eq:entropy-area}
H(X)
=
\int_0^1 \sum_{i=1}^n
\frac{\partial}{\partial t_i} H(X_i\mid Y(\mathbf{t}))\big|_{\mathbf{t}=\gamma(\tau)}
\gamma_i^\prime(\tau)\,d\tau .
\end{align}
For a uniform path, where $\gamma_i(\tau)=\tau$ for all $i=1,\ldots, n$, the same conservation law becomes
\begin{align}
\label{eq:uniform-path-area}
H(X)
=
\int_0^1 
\sum_{i=1}^n \frac{\partial}{\partial t_i} H(X_i\mid Y(\mathbf{t}))\big|_{\mathbf{t}=\gamma(\tau)}\,d\tau,
\end{align}
so the general area theorem specializes to a single shared noise schedule across variables. Throughout the paper, we focus on a uniform path across variables, meaning a single scalar schedule $t_1=\cdots=t_n$, even though the formulation covers nonuniform schedules across variables.

%% file: sections/mismatched_exit.tex
\section{Mismatched Cross-Entropy}
This section derives the mismatched GEXIT identity, then specializes it to several choices of forward noise processes.
\subsection{Mismatched GEXIT Formula}
Let $P_X, Q_X$ be two distributions on $\mathcal X^n$. Let $P_{X, Y(\mathbf t)}$ and $Q_{X, Y(\mathbf t)}$ be the joint laws under $P_X$ and $Q_X$ respectively, with the same admissible channel family $\mathcal W$ as in Definition~\ref{def:admissible-channel-family}.
To make the dependence on the noise vector explicit, write the
$Q$-extrinsic prior as
$Q_{X_i\mid Y_{-i}(\mathbf t_{-i})}(x_i\mid y_{-i})$,
which depends on the noised side information $Y_{-i}(\mathbf t_{-i})$ but
has no dependence on the local noise parameter $t_i$.  The corresponding
$Q$-posterior is
\begin{align}
\label{eq:base-posterior-factorization}
Q_{X_i\mid Y(\mathbf t)}(x_i\mid y)
&=
\frac{
Q_{X_i\mid Y_{-i}(\mathbf t_{-i})}(x_i\mid y_{-i})
W_{t_i}(y_i\mid x_i)
}{
\sum_{u\in\mathcal X}
Q_{X_i\mid Y_{-i}(\mathbf t_{-i})}(u\mid y_{-i})
W_{t_i}(y_i\mid u)
}.
\end{align}
The following theorem generalizes the GEXIT derivative to this mismatched
posterior.

\begin{theorem}[Mismatched GEXIT derivative]
\label{thm:gexit-derivative}
Let $P_X$ and $Q_X$ be two distributions on $\mathcal X^n$.
Let $P_{X,Y(\mathbf t)}$ and $Q_{X,Y(\mathbf t)}$ denote the joint laws
obtained by passing $P_X$ and $Q_X$, respectively, through the same
memoryless channel family $\mathcal W$.
Fix $\mathbf t\in\operatorname{int}(\mathcal{T})^n$. Assume $Q_{X_i\mid Y}$ is
positive on the $P$-support, has finite $P$-log loss, and satisfies
Convention (C2) in Appendix~\ref{app:proofs}. Then, we have
\begingroup
\small
\begin{align}
\label{eq:gexit-derivative}
\frac{\partial}{\partial t_i}
\CE\!\left(P_{X_i\mid Y},Q_{X_i\mid Y}\right)
&=\E\!\left[-S_{t_i}(X_i,Y_i)\log Q_{X_i\mid Y}(X_i\mid Y)+
\E_{X'\sim Q_{X_i\mid Y}(\cdot\mid Y)}[S_{t_i}(X',Y_i)\mid Y]\right].
\end{align}
\endgroup

\end{theorem}

The first term in \eqref{eq:gexit-derivative} is a log-likelihood term weighted by the channel score, while the second term is a posterior mismatch term. In the matched case,
the second term averages to zero and the identity reduces to \eqref{eq:gexit-score}. The proof is given in
Appendix~\ref{subsec:gexit-derivative-proof}.
The local implication of Theorem~\ref{thm:gexit-derivative} is that the CE
derivative can be evaluated from two local quantities: the marginal
posteriors $Q_{X_i\mid Y(\mathbf t)}$ for all $i, \mathbf{t}$, and the known channel score function $S_{t}(x, y)$.  Thus, computing the infinitesimal derivative does not require
evaluating a full joint posterior over $X$.

To turn the local derivative identity into a global conservation law, we integrate \eqref{eq:gexit-derivative} along a path $\gamma\colon(0,1)\to\operatorname{int}(\mathcal{T})^n$ in the noise-parameter space.
Let $F(\mathbf t)\triangleq\CE(P_{X\mid Y(\mathbf t)},Q_{X\mid Y(\mathbf t)})$.
Convention (C2) in Appendix~\ref{app:proofs} implies that
$F(\gamma(\tau))$ is absolutely continuous on compact subintervals of $(0,1)$.
Assume also that $\gamma$ connects a clean endpoint to an uninformative endpoint so
\begin{align*}
\lim_{\tau\downarrow 0}F(\gamma(\tau))=0,
\qquad
\lim_{\tau\uparrow 1}F(\gamma(\tau))=\CE(P_X,Q_X).
\end{align*}
Then the mismatched GEXIT conservation law is
\begin{align}
\label{eq:gexit-area-general}
\CE(P_X,Q_X)
&=
\int_0^1
\sum_{i=1}^n
\gamma_i^\prime(\tau)
\frac{\partial}{\partial t_i}
\CE\!\left(
P_{X_i\mid Y(\mathbf t)},
Q_{X_i\mid Y(\mathbf t)}
\right)
\Big|_{\mathbf t=\gamma(\tau)}
\,d\tau .
\end{align}
The proof of \eqref{eq:gexit-area-general} is given in
Appendix~\ref{subsec:gexit-area-proof}.

\subsection{Special cases}
The general identity above yields simple evaluation curves for admissible channel families. Here, we apply it to well-known channels that are commonly used for diffusion models \cite{Austin2021D3PM,Ho2020DDPM}.

\paragraph{Masked diffusion.}
For erasure rate $t\in[0,1]$, the masking channel is
\begin{align*}
W_t(y\mid x)
&=
\begin{cases}
1-t, & y=x,\\
t, & y=\texttt{[MASK]}.
\end{cases}
\end{align*}
For coordinate $i$, define the masked and observed local CEs by
\begin{align*}
\CE_i^{\mathrm{mask}}(\mathbf{t})
&\triangleq
\CE \big(
P_{X_i\mid Y_{-i}},
Q_{X_i\mid Y_{-i}}
\big),
&
\CE_i^{\mathrm{obs}}(\mathbf{t})
&\triangleq
\CE \big(
P_{X_i\mid Y_{-i},Y_i=X_i},
Q_{X_i\mid Y_{-i},Y_i=X_i}
\big).
\end{align*}
\begin{proposition}[masked GEXIT derivative]
\label{thm:exit-derivative}
For the masking channel, the local CE derivative is the excess CE incurred
when the $i$th coordinate is hidden rather than revealed:
\begin{align}
\label{eq:exit-derivative}
\frac{\partial}{\partial t_i}
\CE \big(
P_{X_i\mid Y(\mathbf{t})},
\;Q_{X_i\mid Y(\mathbf{t})}
\big)
=
\CE_i^{\mathrm{mask}}(\mathbf{t})-\CE_i^{\mathrm{obs}}(\mathbf{t}).
\end{align}
\end{proposition}

\input{sections/qary_symmetric}
\input{sections/mismatched_immse}

%% file: sections/qary_symmetric.tex
\paragraph{Uniform channel.}
Let $q=|\mathcal{X}|$ and let $s\in[0,1]$ denote degradation progress. Define
$a(s)\triangleq 1-s\left(1-\frac{1}{q}\right)$ to be the probability that the output equals the input symbol. The uniform channel is defined by
\begin{align*}
W_s(y\mid x)
&=
\begin{cases}
a(s), & y=x,\\[0.3em]
\dfrac{s}{q}, & y\neq x.
\end{cases}
\end{align*}
At $s=0$ this channel recovers the clean symbol, while $s=1$ makes the output
uniform and therefore uninformative. For coordinate $i$, define the
correct-output and error-output CEs by
\begingroup
\small
\begin{align*}
\CE^{\mathrm{cor}}_i(\mathbf{s})
&\triangleq
\E\!\left[
-\log Q_{X_i\mid Y}(X_i\mid Y)\mid Y_i=X_i
\right],
&
\CE^{\mathrm{err}}_i(\mathbf{s})
&\triangleq
\E\!\left[
-\log Q_{X_i\mid Y}(X_i\mid Y)\mid Y_i\neq X_i
\right].
\end{align*}
\endgroup
\begin{proposition}[uniform-channel GEXIT derivative]
\label{thm:qary-gexit-derivative}
For the uniform channel parameterized by degradation progress $s_i$,
write $a_i=a(s_i)$. The local CE derivative is given by
\begingroup
\small
\begin{align}
\label{eq:qary-gexit-derivative}
\frac{\partial}{\partial s_i}
\CE\big(P_{X_i\mid Y},Q_{X_i\mid Y}\big)
&=
\left(1-\frac{1}{q}\right)\!\Bigg[
\CE^{\mathrm{err}}_i(\mathbf{s})-\CE^{\mathrm{cor}}_i(\mathbf{s})
+\frac{1}{1-a_i}\left(1-\frac{1}{a_i}
\E\!\left[
Q_{X_i\mid Y}(Y_i\mid Y)
\right]\right)
\Bigg]
.
\end{align}
\endgroup
\end{proposition}

%% file: sections/mismatched_immse.tex
\paragraph{Gaussian diffusion.}
For Gaussian noise we use the native signal-to-noise ratio (SNR) parameter
$\rho>0$ and an injective embedding $f\colon\mathcal{X}\to\mathbb{R}^d$. For
$X=(X_1,\ldots,X_n)$, write $U_i=f(X_i)$. The Gaussian channel acts
coordinatewise in the embedding space as
\begin{align*}
Y_i(\rho)=U_i+\rho^{-1/2}Z_i,
\qquad
Z_i\sim\mathcal{N}(0,I_d),
\qquad i=1,\ldots,n
\end{align*}
and we write $Y(\rho)=\bigl(Y_1(\rho),\ldots,Y_n(\rho)\bigr)$
with kernel $W_{\rho}(y_i\mid x_i)$ on each coordinate. The injectivity
assumption ensures that the continuous representation preserves the discrete
symbol. The exact conservation law is unchanged by the choice of injective
embedding, but under finite-capacity models the embedding can simplify or
complicate the denoising problem.

As in the discrete cases, the model supplies the extrinsic component of the
marginal posterior and the known Gaussian likelihood is inserted analytically:
\begin{align*}
Q_{X_i\mid Y(\rho)}(x_i\mid y)
&\propto
Q_{X_i\mid Y_{-i}(\rho)}(x_i\mid y_{-i})\,W_\rho(y_i\mid x_i).
\end{align*}
Define the corresponding model marginal posterior mean
\begin{align*}
m_{Q,i}(y)
&\triangleq
\E_{X'\sim Q_{X_i\mid Y(\rho)}(\cdot\mid y)}
\!\left[f(X')\right].
\end{align*}
The mismatched mean-square error (MSE) on the full sequence is then
\begin{align*}
 \mathrm{MSE}_{P_{X\mid Y(\rho)},\tilde{Q}_{X\mid Y(\rho)}}(\rho)
&\triangleq
\E_P\!\left[
    \sum_{i=1}^n
\left\|
f(X_i)
 -
m_{Q,i}(Y(\rho))
\right\|^2
\right],
\end{align*}
where $\tilde{Q}_{X|Y(\rho)}(x|y) = \prod_{i=1}^n Q_{X_i|Y(\rho)}(x_i|y)$.
\begin{proposition}[Gaussian CE area formula]
\label{cor:gaussian-ce-derivative}
If the model posteriors are modeled using the Bayes factorization in
\eqref{eq:base-posterior-factorization}, then the CE area over the
native SNR axis is one half of the mismatched MSE in the embedding space:
\begin{align}
\label{eq:immse-eval}
\CE\!\big(P_X,Q_X\big)
&=
\frac{1}{2}\int_0^\infty
\mathrm{MSE}_{P_{X\mid Y(\rho)},\tilde{Q}_{X\mid Y(\rho)}}(\rho)\,d\rho .
\end{align}
\end{proposition}
The mean in this formula is
the expectation under the channel-tilted categorical posterior. We give two
proofs: Appendix~\ref{subsec:gaussian-ce-proof} derives the identity from our
marginal GEXIT formula, while
Appendix~\ref{subsec:gaussian-ce-verdu-proof} gives a short derivation from
I--MMSE and Verd\'u's mismatched estimation formula.  In the matched case,
Proposition~\ref{cor:gaussian-ce-derivative} reduces to the Gaussian I--MMSE
identity \cite{Guo2005IMmse}; the connection between the two derivations is
discussed in Appendix~\ref{subsec:gaussian-ce-verdu-connection}.
The masked and uniform-channel specializations are proved in
Appendices~\ref{subsec:masked-exit-proof} and \ref{subsec:qary-gexit-proof}.

%% file: sections/implementation.tex
\section{Implementation}
\label{sec:implementation}

In this section, we replace $Q$ related quantities with a neural network approximations and denote the resulting learned model by $P^\theta$.

\subsection{Extrinsic-Only Posterior Modeling}
\label{subsec:extrinsic-posterior-modeling}
\paragraph{Channel-Aware Logits Tilting.}
For every admissible channel, the corruption kernel is known exactly. The model
therefore does not need to learn the full posterior from scratch. The
marginal posterior satisfies
\begin{align}
\label{eq:extrinsic-factorization}
P^\theta_{X_i\mid Y}(x_i\mid y)
&\propto
P^\theta_{X_i\mid Y_{-i}}(x_i\mid y_{-i})\,W_{t_i}(y_i\mid x_i).
\end{align}
Thus, we parameterize only the extrinsic prior term
$P^\theta_{X_i\mid Y_{-i}}(x_i\mid y_{-i})$
and use the known channel kernel to form the full marginal posterior. Concretely, the
network produces logits $\left\{\ell_i^\theta(x_i;y)\right\}_{x_i\in \mathcal X}$
and the channel kernel tilts those logits into the marginal
posterior:
\begin{align}
\label{eq:tilted-posterior}
P^\theta_{X_i\mid Y}(x_i\mid y)
&=
\mathrm{Softmax}\!\left(
\ell_i^\theta(x_i;y)+\log W_{t_i}(y_i\mid x_i)
\right)
,
\end{align}
where the softmax is taken over $x_i\in\mathcal X$.
Ideally, the network should compute the logits based only on $y_{-i}$ but, in practice, transformers cannot be restricted to ignore the local observation $y_i$ for all symbols at one forward pass. Hence, we use the full observation $y$ as input to the network and rely on the channel tilting to encourage the network to learn only the extrinsic prior.

\subsection{Channel-Shared Cross-Entropy Training}
\label{subsec:channel-shared-training}

The conservation laws depend on two ingredients: the known channel score and
the model marginal posteriors. Thus, training only needs to estimate the posterior term, which is learned by a CE objective for all admissible channels:
\begin{align}
\label{eq:shared-token-ce-loss}
\mathcal{L}_{\mathrm{CE}}(\theta)
&=
\int_0^1
\E_{X,Y(\gamma(\tau))}
\left[
-\frac{1}{n}\sum_{i=1}^n
\log P^\theta_{X_i\mid Y(\gamma(\tau))}
\bigl(X_i\mid Y(\gamma(\tau))\bigr)
\right]\,d\tau,
\end{align}
where $P^\theta_{X_i\mid Y}$ is the channel-tilted marginal posterior from
\eqref{eq:tilted-posterior} along the noising path $\gamma(\tau)$. The loss is
minimized when the model marginal posteriors match the data marginal posteriors
across noise levels. The exact CE identities assume that the learned marginals
are compatible with some joint law $Q_X$; in implementation we do not enforce
this global consistency, and use the resulting conservation-law area as a
likelihood-scale evaluation surrogate. The conservation law is then used to
evaluate the CE on training and validation sets, not as a separate training
objective.

\subsection{Avoiding Embedding Collapse in Gaussian Diffusion}
Our Gaussian implementation avoids training a regressor for the clean
embedding. The network instead predicts a categorical distribution, tilts it by
the Gaussian channel likelihood, and computes the embedding-space posterior mean
from that distribution only when evaluating the I-MMSE curve. This avoids
minimizing an embedding-space MSE in which both the target embedding and the
posterior-mean estimate depend on learned parameters. This design helps address
the embedding-collapse problem observed when applying Gaussian noise to
discrete data
\cite{Li2022DiffusionLM,Gao2024Difformer,Gulrajani2024LikelihoodDiffusionLM,Nguyen2025VQLCMD}.

%% file: sections/experiments.tex
\section{Experiments}
This section validates the conservation laws numerically. We start with synthetic Markov data, where the exact entropy rate and reference GEXIT curves can be computed analytically, and then evaluate the same methodology on text and images.

\subsection{Markov Data}
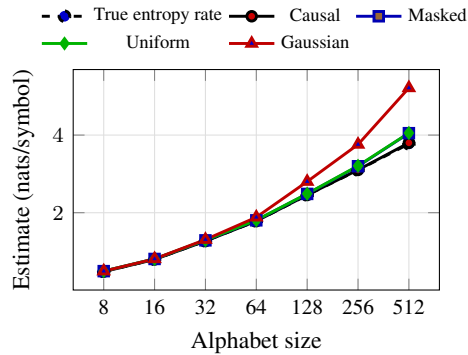
\begin{wrapfigure}{r}{0.46\linewidth}
\vspace{-0em}
\centering
\input{shared/figures/tikz/exp_markov_objective_estimates_d128}
\vspace{-1.8em}
\caption{\small Synthetic Markov data: estimated and true entropy as alphabet size increases.}
\label{fig:exp_markov_objective_estimates_d128}
\vspace{-0.75em}
\end{wrapfigure}
We begin with synthetic first-order Markov sources, where the entropy rate and
reference GEXIT curves are available analytically. This setting lets us compare
conservation-law area estimates directly against ground truth. Figure~\ref{fig:exp_markov_objective_estimates_d128}
compares the true entropy rate with estimates from a causal transformer and
from denoising transformers trained with masked, uniform, and
Gaussian noising paths, across alphabet sizes from 8 to 512.
All models use the same hyperparameters: hidden size 256, sequence length 256,
8 layers, and 8 attention heads. The autoregressive baseline uses causal
masking, while the denoising transformers use self-masking (excluding each
position's own input from its attention context) and
channel-aware logit-tilting construction in Section~\ref{subsec:extrinsic-posterior-modeling}
and the shared cross-entropy loss in Section~\ref{subsec:channel-shared-training}.
Additional implementation details and experiments are deferred to Appendix~\ref{app:markov}.

The causal transformer and the GEXIT estimates obtained with masked and uniform
noise remain close to the true entropy rate across alphabet sizes. In contrast,
the Gaussian-noise estimate lies above the true entropy rate, with a gap that
increases alphabet size. Thus, although the conservation
law does not depend on the channel, the model's finite-capacity approximation error does.

\subsection{Language Modeling on \texttt{text8}}
\begin{wraptable}{r}{0.48\linewidth}
\vspace{-2em}
\centering
\scriptsize
\setlength{\tabcolsep}{3pt}
\caption{Diffusion-model results on \texttt{text8}, reported in bits per character (BPC).}
\vspace{0.5em}
\label{tab:text8_diffusion_summary}
\begin{tabular}{lcc}
\toprule
Model & External data & BPC ($\downarrow$) \\
\midrule
masked (ours) & No & 1.424 \\
uniform (ours) & No & 1.513 \\
Gaussian (ours) & No & 2.087 \\
\midrule
D3PM mask/absorbing \cite{Austin2021D3PM} & No & $\leq 1.45 \pm 0.02$ \\
SEDD Absorb \cite{Lou2023DiscreteDiffusionLM} & No & $\leq 1.39$ \\
MDLM \cite{Sahoo2024MDLM} & Yes & $\leq 1.40$ \\
MD4 \cite{Shi2024MD4} & Yes & $\leq 1.37$ \\
EDLM \cite{Xu2025EDLM} & Yes & $\leq 1.24$ \\
\bottomrule
\end{tabular}
\vspace{-1.2em}
\end{wraptable}
We next turn to the character-level \texttt{text8} benchmark \cite{Mahoney2011TextData}. We use the raw character stream without preprocessing and train standard diffusion transformers with context length 256, hidden size 768, 12 layers, and 12 attention heads. The diffusion models use adaptive layer normalization with zero initialization (adaLN-Zero) timestep conditioning \cite{PeeblesXie2023DiT}, and we follow the D3PM training setup and hyperparameter choices \cite{Austin2021D3PM}. As in the Markov experiments, all objectives are trained with the same channel-shared cross-entropy loss from Section~\ref{subsec:channel-shared-training}, and likelihood is evaluated through the corresponding conservation law. Table~\ref{tab:text8_diffusion_summary} summarizes the resulting diffusion-model bits-per-character (BPC) values; details are deferred to Appendix~\ref{app:lm-repro}.
The baseline acronyms are score entropy discrete diffusion (SEDD), masked diffusion language model (MDLM), masked discrete diffusion model (MD4), and energy-based diffusion language model (EDLM).

Figure~\ref{fig:exp_text8_eval_curves} shows the GEXIT curves used to compute these likelihood estimates. The left and middle panels keep each objective on its native evaluation axis: masked and uniform are plotted against their discrete-channel noise levels, while Gaussian is plotted against SNR $\rho$. The right panel puts the three objectives on a common uncertainty scale, $H(X\mid Y)$, where $X$ is taken to be uniform over the 256 byte symbols. For example, an erasure rate of $0.1$ gives $H(X\mid Y)=0.1\log 256=0.555$ nats/token; for the uniform channel, a substitution probability of $0.1$ gives $H(X\mid Y)=h_2(0.9)+0.1\log 255=0.879$ nats/token. For the Gaussian curve, the uncertainty is estimated for the Gaussian channel induced by the learned embedding map $f_\theta$. This yields an estimate of $H(X\mid Y)\approx 0.009$ nats/token for SNR $\rho=0.1$. This allows us to plot all curves with a common information-based axis.

\begin{figure}[!t]
\centering
\input{shared/figures/tikz/exp_text8_eval_curves}
\caption{\texttt{text8} evaluation curves for masked, uniform, and Gaussian noise. The left and middle panels show the native evaluation axes, and the right panel reparametrizes the corresponding cumulative areas by the common uncertainty scale.}
\label{fig:exp_text8_eval_curves}
\end{figure}
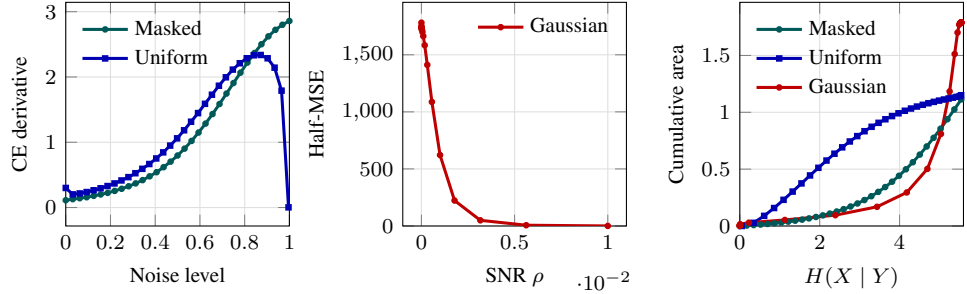

\subsection{CIFAR-10}
For CIFAR-10 \cite{Krizhevsky2009CIFAR10}, we train masked, uniform, and Gaussian diffusion models on $32\times32$ RGB images represented as 8-bit subpixel tokens. All three models use the same DDPM U-Net implementation from \cite{Ho2020DDPM}, and samples are generated with 1000 reverse-diffusion steps. The masked and uniform models are sampled with the corresponding D3PM categorical reverse-posterior samplers \cite{Austin2021D3PM}, whereas the Gaussian model is sampled with a score-based stochastic differential equation (score-SDE) reverse process \cite{Song2021SDE}. Architecture and optimization details are deferred to Appendix~\ref{app:cifar-repro}. Table~\ref{tab:cifar_diffusion_bpd} compares conservation-law likelihood estimates with published likelihood baselines, where maximum-likelihood Score SDE is abbreviated as ML Score SDE. Table~\ref{tab:cifar_exit_generation_scores} reports Fr\'echet Inception Distance (FID) \cite{Heusel2017FID} and Inception Score (IS) \cite{Salimans2016ImprovedGAN} for the corresponding samplers alongside D3PM \cite{Austin2021D3PM}, DDPM \cite{Ho2020DDPM}, and Score SDE \cite{Song2021SDE}. Figure~\ref{fig:cifar_exit_generation_traces} shows representative generations and reverse-process traces for the same three noising paths, complementing the FID and IS results in Table~\ref{tab:cifar_exit_generation_scores}.
For sampling schedules, we evaluate linear, cosine, and equal-information schedules, following the schedule families used in D3PM \cite{Austin2021D3PM}. The equal-information schedule uses the conservation-law curve to partition the noise-parameter range into equal-area intervals; details are given in Appendix~\ref{app:cifar-schedules}.

The image results highlight the dependence of finite-model performance on the channel family. Across modalities, no single noising path dominates: masked noise gives the best \texttt{text8} likelihood in Table~\ref{tab:text8_diffusion_summary}, whereas Gaussian noise gives the best CIFAR-10 likelihood and sample-quality scores among our trained models. The uniform channel performs worse than the masked and Gaussian channels on CIFAR-10 under the reported metrics, despite using the same model class and training protocol. 
\begin{table}[t]
\centering
\caption{CIFAR-10 sample quality and likelihood. Left: FID and Inception Score (IS); lower FID and higher IS are better. Right: likelihood in bits per dimension (bpd); lower is better.}
\label{tab:cifar_exit_generation_scores}
\label{tab:cifar_diffusion_bpd}
\vspace{-2mm}
\begin{minipage}[t]{0.58\linewidth}
\centering
\scriptsize
\setlength{\tabcolsep}{2pt}
\textbf{Sample Quality}
\vspace{0.25em}

\begin{tabular}{@{}lccc@{}}
\toprule
Noise & Schedule & FID $\downarrow$ & IS $\uparrow$ \\
\midrule
uniform (ours) & cosine & $59.51$ & $6.21\pm0.06$ \\
masked (ours) & equal-info & $30.69$ & $6.87\pm0.08$ \\
Gaussian (ours) & log-uniform & $26.66$ & $7.34\pm0.08$ \\
\midrule
D3PM uniform \cite{Austin2021D3PM} & cosine & $51.27\pm2.15$ & $5.99\pm0.14$ \\
D3PM absorbing \cite{Austin2021D3PM} & mutual-info & $41.28\pm0.65$ & $6.26\pm0.10$ \\
D3PM Gauss + logistic \cite{Austin2021D3PM} & linear & $7.34\pm0.19$ & $8.56\pm0.10$ \\
\midrule
DDPM \cite{Ho2020DDPM} & linear-$\beta$ & 3.17 & 9.46 \\ 
Score SDE \cite{Song2021SDE} & geometric-$\sigma$ & 2.20 & 9.89 \\
\bottomrule
\end{tabular}
\end{minipage}
\begin{minipage}[t]{0.4\linewidth}
\centering
\scriptsize
\setlength{\tabcolsep}{10pt}
\textbf{Likelihood}
\vspace{0.25em}

\begin{tabular}{@{}lc@{}}
\toprule
Model & bpd $\downarrow$ \\
\midrule
DDPM \cite{Ho2020DDPM} & $\le 3.70$ \\
Improved DDPM \cite{NicholDhariwal2021ImprovedDDPM} & $\le 3.40$ \\
Score SDE \cite{Song2021SDE} & $\le 2.99$ \\
ML Score SDE \cite{Song2021MLScore} & $\le 2.99$ \\
\midrule
D3PM Gauss + logistic \cite{Austin2021D3PM} & $\le 3.435\pm0.007$ \\
\midrule
uniform  & 3.66 \\
masked   & 3.31 \\
Gaussian & 3.24 \\
\bottomrule
\end{tabular}
\end{minipage}
\end{table}

\begin{figure}[t]
\centering
\includegraphics[width=\linewidth]{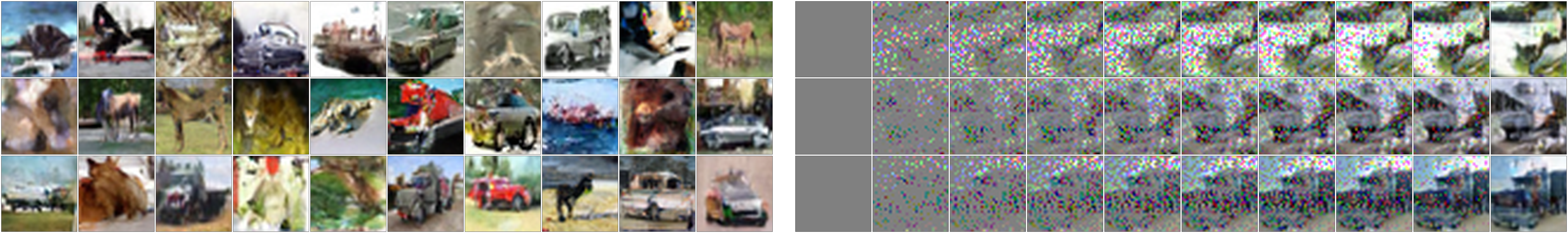}\\[0.35em]
\includegraphics[width=\linewidth]{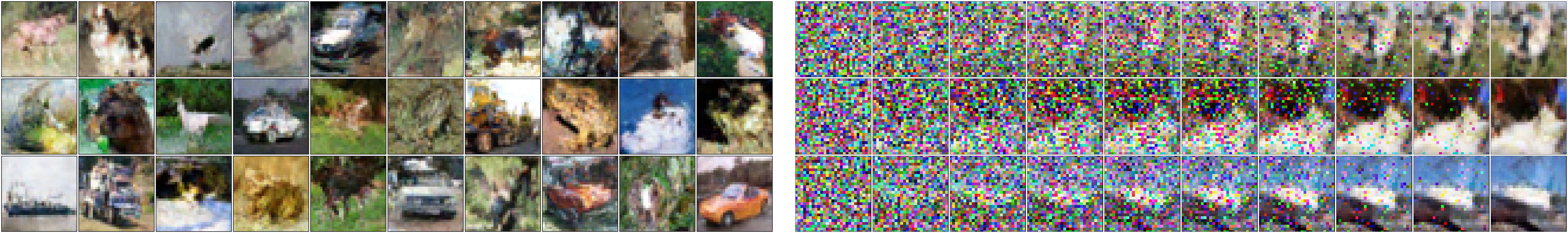}\\[0.35em]
\includegraphics[width=\linewidth]{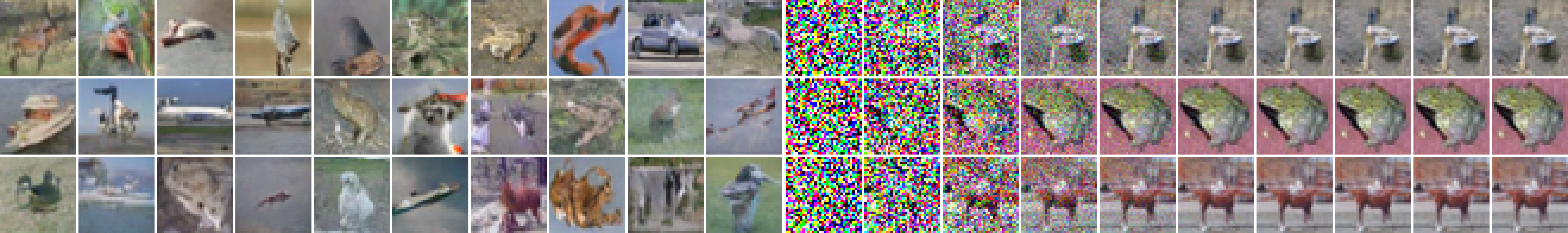}
\caption{CIFAR-10 generations and diffusion traces for three noising paths. Rows show, from top to bottom, masked, uniform \cite{Austin2021D3PM}, and Gaussian noise \cite{Ho2020DDPM,Song2021SDE}.}
\label{fig:cifar_exit_generation_traces}
\end{figure}

%% file: shared/figures/tikz/exp_markov_objective_estimates_d128.tex
\begin{tikzpicture}
\begin{axis}[
width=\linewidth,
height=0.70\linewidth,
xmode=log,
log basis x=2,
xlabel={Alphabet size},
ylabel={Estimate (nats/symbol)},
xtick={8,16,32,64,128,256,512},
xticklabels={8,16,32,64,128,256,512},
tick label style={font=\footnotesize},
label style={font=\small},
grid=both,
major grid style={draw=gray!25},
minor grid style={draw=gray!15},
legend columns=3,
legend style={draw=none, fill=none, at={(0.5,1.03)}, anchor=south, font=\scriptsize},
]
\addplot+[line width=1.1pt, dashed, color=black]
table [x=K, y=entropy_nats, col sep=comma] {shared/figures/data/markov_synthetic/objective_estimates_d128.csv};
\addlegendentry{True entropy rate}

\addplot+[line width=1.1pt, mark=*, mark size=1.8pt, color=black]
table [x=K, y=causal_ce_nats, col sep=comma] {shared/figures/data/markov_synthetic/objective_estimates_d128.csv};
\addlegendentry{Causal}

\addplot+[line width=1.1pt, mark=square*, mark size=1.8pt, color=blue!70!black]
table [x=K, y=exit_area_nats, col sep=comma] {shared/figures/data/markov_synthetic/objective_estimates_d128.csv};
\addlegendentry{Masked}
\addplot+[line width=1.1pt, mark=diamond*, mark size=1.8pt, color=green!70!black]
table [x=K, y=qsc, col sep=comma] {shared/figures/data/markov_synthetic/objective_estimates_d128.csv};
\addlegendentry{Uniform}

\addplot+[line width=1.1pt, mark=triangle*, mark size=2.1pt, color=red!75!black]
table [x=K, y=immse_area_nats, col sep=comma] {shared/figures/data/markov_synthetic/objective_estimates_d128.csv};
\addlegendentry{Gaussian}
\end{axis}
\end{tikzpicture}

%% file: shared/figures/tikz/exp_text8_eval_curves.tex
\begin{tikzpicture}
\begin{groupplot}[
group style={group size=3 by 1, horizontal sep=1.5cm},
width=0.325\linewidth,
height=0.325\linewidth,
grid=both,
major grid style={draw=gray!25},
minor grid style={draw=gray!15},
tick label style={font=\footnotesize},
label style={font=\footnotesize},
title style={font=\footnotesize},
legend style={draw=none, fill=none, font=\footnotesize},
every axis plot/.append style={mark size=0.65pt},
]
\nextgroupplot[
xlabel={Noise level},
ylabel={CE derivative},
xmin=0,
xmax=1,
legend cell align=left,
legend pos=north west,
]
\addplot+[line width=1.15pt, color=teal!70!black]
coordinates {
(0.0,0.11152732849963903)
(0.03125,0.12837923599643591)
(0.0625,0.14248955524575876)
(0.09375,0.15768724195731748)
(0.125,0.17885554164446157)
(0.15625,0.20024030093049805)
(0.1875,0.22475466930147942)
(0.21875,0.2546758964804778)
(0.25,0.28717031478185207)
(0.28125,0.3259396747014216)
(0.3125,0.3699368216856995)
(0.34375,0.41994202675409587)
(0.375,0.4780528799656533)
(0.40625,0.5415870199516619)
(0.4375,0.618205964648185)
(0.46875,0.7031838576458531)
(0.5,0.7967666167684975)
(0.53125,0.9024343200648903)
(0.5625,1.020219123883393)
(0.59375,1.1487265468743415)
(0.625,1.2855991713126425)
(0.65625,1.43240066994896)
(0.6875,1.5858512316701663)
(0.71875,1.7459250053187954)
(0.75,1.9072242416816243)
(0.78125,2.065939550443493)
(0.8125,2.2209788567237694)
(0.84375,2.3667860029575998)
(0.875,2.5016593521216914)
(0.90625,2.62245732696128)
(0.9375,2.7224069212203066)
(0.96875,2.8000105634282337)
(1.0,2.85784863319688)
};
\addlegendentry{Masked}

\addplot+[line width=1.15pt, color=blue!70!black]
coordinates {
(0.99609375,0.0028356025878355506)
(0.9649689453125,1.7904876769677394)
(0.933844140625,2.142566454774072)
(0.9027193359375,2.286297775137448)
(0.87159453125,2.3368402162404873)
(0.8404697265625,2.332211515601462)
(0.809344921875,2.2869062562725624)
(0.7782201171875,2.2102104653353822)
(0.7470953125,2.1129549576488618)
(0.7159705078125,1.9950023178267653)
(0.684845703125,1.8658886142645055)
(0.6537208984375,1.7282731581019364)
(0.62259609375,1.5890081282798378)
(0.5914712890625,1.4466903736508934)
(0.560346484375,1.3145553844109434)
(0.5292216796875,1.1834431030259087)
(0.498096875,1.06028806598379)
(0.4669720703125,0.9489198805412681)
(0.435847265625,0.8451966801842634)
(0.4047224609375,0.7518491675166835)
(0.37359765625,0.6665897233731183)
(0.3424728515625,0.5931119527086132)
(0.311348046875,0.5255321651725483)
(0.2802232421875,0.46566913428139733)
(0.2490984375,0.41507203449890206)
(0.2179736328125,0.3693304159030768)
(0.186848828125,0.32916231714005806)
(0.1557240234375,0.29471087773478377)
(0.12459921875,0.2645760535442971)
(0.0934744140625,0.23708871488136948)
(0.062349609375,0.21630774161703586)
(0.0312248046875,0.20184421291387744)
(0.0001,0.3000106683556084)
};
\addlegendentry{Uniform}

\nextgroupplot[
xlabel={SNR $\rho$},
ylabel={Half-MSE},
ymin=0,
]
\addplot+[line width=1.15pt, color=red!75!black]
coordinates {
(1e-06,1781.6130887225356)
(1.778279410038923e-06,1752.4821093459996)
(3.162277660168379e-06,1741.6188859092397)
(5.623413251903491e-06,1736.7942835605795)
(1e-05,1733.2332273103095)
(1.778279410038923e-05,1727.989697886883)
(3.1622776601683795e-05,1718.1617405827794)
(5.623413251903491e-05,1699.7857754829263)
(0.0001,1663.491081704596)
(0.00017782794100389227,1583.9406300963153)
(0.00031622776601683794,1411.8935649501313)
(0.0005623413251903491,1087.3359872037613)
(0.001,621.7126208137339)
(0.0017782794100389228,223.16069438013025)
(0.0031622776601683794,49.68580377526098)
(0.005623413251903491,6.687362392017395)
(0.01,0.43363049770005135)
};
\addlegendentry{Gaussian}

\nextgroupplot[
xlabel={$H(X\mid Y)$},
ylabel={Cumulative area},
xmin=0,
xmax=5.6,
ymin=0,
legend cell align=left,
legend pos=north west,
]
\addplot+[line width=1.15pt, color=teal!70!black]
coordinates {
(0,0)
(0.173286795139986,0.00374854007025117)
(0.346573590279973,0.00798086493341046)
(0.519860385419959,0.0126711273897085)
(0.693147180559945,0.0179296083834863)
(0.866433975699932,0.0238529809237201)
(1.03972077083992,0.0304935273335947)
(1.2130075659799,0.0379846299239378)
(1.38629436111989,0.0464509769749117)
(1.55958115625988,0.0560308205605879)
(1.73286795139986,0.0669038908166366)
(1.90615474653985,0.0792457478235084)
(2.07944154167984,0.0932769182410045)
(2.25272833681982,0.109208791677213)
(2.42601513195981,0.127330557061585)
(2.59930192709979,0.14797727303493)
(2.77258872223978,0.171413999197654)
(2.94587551737977,0.197964013835675)
(3.11916231251975,0.228005473897367)
(3.29244910765974,0.261895250002957)
(3.46573590279973,0.299931589349629)
(3.63902269793971,0.342400336869341)
(3.8123094930797,0.38956052283214)
(3.98559628821969,0.441619526535093)
(4.15888308335967,0.498699983519474)
(4.33216987849966,0.560780667771429)
(4.50545667363964,0.627763767883418)
(4.67874346877963,0.699447593815939)
(4.85203026391962,0.775517052489053)
(5.0253170590596,0.855581375599724)
(5.19860385419959,0.939094879477562)
(5.37189064933958,1.0253826526752)
(5.54517744447956,1.11378670262246)
};
\addlegendentry{Masked}

\addplot+[line width=1.15pt, color=blue!70!black]
coordinates {
(0.0472050207676865,0)
(0.345933772621155,0.0279084184088571)
(0.61016692721761,0.0891161892467715)
(0.858124361248868,0.158039956318497)
(1.09487283160019,0.229987089830997)
(1.32288943267953,0.302648751445704)
(1.54364147577043,0.374533320684774)
(1.75809226600706,0.444519260493242)
(1.96691535371984,0.511798100204311)
(2.17060113607095,0.575727784136198)
(2.36951581969422,0.635812522226738)
(2.56393653999019,0.69174631381683)
(2.75407340501162,0.743371179882896)
(2.94008382553245,0.790613941364261)
(3.12208199637683,0.833585558821193)
(3.30014513809172,0.872460416336764)
(3.47431742653652,0.907378263540105)
(3.64461213824611,0.938646365996188)
(3.81101228083453,0.966567129765784)
(3.9734697849705,0.991420999809121)
(4.13190315961648,1.01349531652917)
(4.28619331432637,1.03309930084545)
(4.43617698143203,1.05050809068739)
(4.58163674625798,1.06593356411314)
(4.72228596201387,1.0796400125424)
(4.85774546561519,1.09184719907498)
(4.98750628021958,1.10271742402078)
(5.11086650775813,1.11242638969091)
(5.22681597853607,1.12113023794107)
(5.3338006722908,1.128937346909)
(5.4291521278721,1.13599328498625)
(5.50721294082566,1.14250073394349)
(5.54173044784094,1.15031080152398)
};
\addlegendentry{Uniform}

\addplot+[line width=1.15pt, color=red!75!black]
coordinates {
(-0,0)
(1.10208487569372e-36,0.00137525676288708)
(5.18883218326797e-17,0.00379317159449151)
(3.35297700075898e-05,0.00807359482156266)
(0.00891987420618534,0.0156670330313666)
(0.226188838481903,0.0291360257125449)
(1.12737274169922,0.0529833635151605)
(2.39240670204163,0.0950435249275201)
(3.4347517490387,0.168641889544253)
(4.18077325820923,0.295012351384358)
(4.6965765953064,0.502323815565472)
(5.03623294830322,0.809870955701579)
(5.25305700302124,1.18386093018669)
(5.37933206558228,1.51263468284004)
(5.45152616500854,1.70144422084055)
(5.49254894256592,1.77081522367709)
(5.51554107666016,1.7863980452343)
(5.52851390838623,1.7881381258541)
(5.53580474853516,1.78823491556043)
(5.53990602493286,1.78824140193336)
(5.54221391677856,1.78824448918348)
(5.54350996017456,1.78824997917681)
(5.54424095153809,1.78825974191892)
(5.54465103149414,1.78827710280218)
(5.54488134384155,1.78830797530344)
};
\addlegendentry{Gaussian}
\end{groupplot}
\end{tikzpicture}

%% file: sections/limitations.tex
\section{Limitations}
\label{sec:limitations}
Our theoretical results assume memoryless degraded channel families and model posteriors that factor through the known corruption channel.
The marginal-posterior locality in our conservation laws is also infinitesimal: while the derivatives depend only on coordinate-wise posterior marginals, finite-step reverse sampling can depend on joint posterior structure, so marginals alone need not determine an exact finite-step sampler.
The empirical scope is also limited. We evaluate synthetic Markov sources, byte-level \texttt{text8}, and CIFAR-10. Due to computational constraints, we did not extend to text corpora with large vocabularies or larger image datasets.
The large-scale neural experiments use one training seed per objective, so the reported sample-quality error bars capture fixed-checkpoint evaluation variability rather than variation across independent retraining runs.
The Gaussian formulation also computes conditional means over the discrete support induced by the data representation. This is adequate for the alphabets considered here, but it may become computationally expensive for very large vocabularies or structured discrete spaces without additional approximation.

%% file: sections/related_work.tex
\section{Related Work}
Recent information-theoretic analyses of diffusion models establish conservation-law identities for specific noise processes. Information-Theoretic Diffusion studies Gaussian diffusion \cite{Kong2023InfoDiffusion}, and Information-Theoretic Discrete Diffusion develops analogous identities for discrete diffusion, including an exact likelihood formula for the masked case \cite{Jeon2025InfoDiscreteDiffusion}. The information-theoretic discrete Poisson diffusion model (ItDPDM) gives a Poisson-channel counterpart \cite{Bhattacharya2025ItDPDM}. Complementary to these conservation-law analyses, \cite{ReevesPfister2025InfoProofsDiffusionSampling} gives a direct discrete-time information-theoretic proof for Gaussian diffusion sampling, bounding path-space KL through step size and conditional-mean estimation error. Our work differs from these approaches by formulating a single channel-path view that covers a broad class of channels. We also analyze the mismatched setting, where the denoiser is learned and need not equal the true posterior, yielding a cross-entropy conservation law. In addition, our derivation shows that the infinite-step diffusion limit can be expressed using marginal posteriors. Prior empirical work on D3PMs also observed that the choice of noise process affects diffusion performance \cite{Austin2021D3PM}. We extend this observation through the GEXIT formula and the associated conservation-law curves.

%% file: sections/conclusion.tex
\section{Conclusion}
We developed a unified mismatched GEXIT framework for diffusion likelihoods under memoryless degraded channel paths. The framework places masked, uniform, and Gaussian corruptions on a common cross-entropy scale and expresses the relevant derivative through a marginal posterior that combines the learned extrinsic prior with the known channel evidence. This perspective separates the conservation law, which is path-invariant for the exact posterior, from the finite-model approximation problem, where the choice of corruption channel and schedule can substantially affect performance.

Empirically, the channel choice produced different likelihood and generation tradeoffs across discrete text and image data. The results show that no single noising family is uniformly dominant under finite-capacity denoisers: masked corruption performs best for \texttt{text8}, whereas Gaussian corruption gives the strongest CIFAR-10 results among our trained models. These observations support treating the forward channel and its path as design variables rather than fixed implementation details. Future work can use the same conservation-law constraint to optimize nonuniform and modality-dependent schedules, including paths that allocate information unevenly across symbols, spatial locations, or noise levels.


%% file: sections/appendix_proofs.tex
\paragraph{Smoothness conventions.}
All derivative identities in this appendix are local in an interior noise
parameter.  For a channel family $W_t$, fix
$t_0\in\operatorname{int}(\mathcal{T})$ and a neighborhood
$I\subset\operatorname{int}(\mathcal{T})$ of $t_0$.  We write
$W_t(y\mid x)$ for the density or mass function with respect to the
counting measure in the discrete case and the Lebesgue measure in the
continuous case.  We use the following conventions.
\begin{enumerate}
\item[(C1)] \textbf{Channel regularity.}
For every $x$, the map $t\mapsto W_t(y\mid x)$ is absolutely continuous
on $I$ for almost every $y$, $\partial_t W_t(y\mid x)$ exists, and there
is an integrable envelope $M_x(y)$ such that
\begin{align*}
\sup_{t\in I}\left|\partial_t W_t(y\mid x)\right|\le M_x(y),
\qquad
\int \partial_t W_t(y\mid x)\,dy=0 .
\end{align*}
The score is $S_t(x,y)=\partial_t W_t(y\mid x)/W_t(y\mid x)$ on the
support of $W_t(\cdot\mid x)$.

\item[(C2)] \textbf{Mismatched CE domination.}
For the mismatched CE derivative, let
$p_i^{s,\mathbf t_{-i}}(x_i,y)$ be the density or mass function of
$(X_i,Y)$ under $P_X$ when the local coordinate is $s$ and the other
coordinates are fixed at $\mathbf t_{-i}$, and let
$q_i^{s,\mathbf t_{-i}}(x_i\mid y)$ be the corresponding $Q$-posterior.
We require $q_i^{s,\mathbf t_{-i}}>0$ on the $P$-support and an
integrable envelope $B_i(x_i,y)$ such that, for $s\in I$,
\begin{align*}
\left|
\partial_s\!\left\{
p_i^{s,\mathbf t_{-i}}(x_i,y)
\log q_i^{s,\mathbf t_{-i}}(x_i\mid y)
\right\}
\right|
\le B_i(x_i,y).
\end{align*}
This condition implies absolute continuity of the local CE on $I$ and
justifies moving the derivative through the expectation by dominated
convergence.
We assume this also applies in the matched case $Q_X = P_X$, where the integrand reduces to the posterior entropy loss $-\log P_{X_i\mid Y}(X_i\mid Y)$.

\item[(C3)] \textbf{Endpoint limits.}
For each area theorem path $\gamma\colon(0,1)\to
\operatorname{int}(\mathcal T)^n$, we assume that the one-sided limits of
the relevant potential along $\gamma$ exist as $\tau\downarrow 0$ and
$\tau\uparrow 1$ and equal the displayed clean and uninformative endpoint
values.  When an endpoint lies on the boundary of $\mathcal T$, expressions
such as $F(\gamma(0))$ and $F(\gamma(1))$ denote these one-sided limits.
\end{enumerate}

\subsection{Proof of Lemma~\ref{lem:gexit-locality}}
\label{subsec:gexit-locality-proof}
\begin{proof}
Fix $i$ and write $Y\equiv Y(\mathbf{t})$ and $Y_{-i}\equiv Y_{-i}(\mathbf{t}_{-i})$. By the chain rule for conditional entropy,
\begin{align*}
\frac{\partial}{\partial t_i} H(X\mid Y)
&=
\frac{\partial}{\partial t_i}
\Big(
H(X_i\mid Y)+H(X_{-i}\mid X_i,Y)
\Big).
\end{align*}
Using the memoryless channel factorization, $Y_i$ is conditionally independent of $(X_{-i},Y_{-i})$ given $X_i$, so
\begin{align*}
H(X_{-i}\mid X_i, Y)
&=
H(X_{-i}\mid X_i, Y_{-i}),
\end{align*}
which does not depend on $t_i$.
Hence,
\begin{align*}
\frac{\partial}{\partial t_i} H(X\mid Y)
&=
\frac{\partial}{\partial t_i} H(X_i\mid Y).
\end{align*}
\end{proof}

\subsection{Proof of the score-form representation}
\label{subsec:gexit-score-proof}
\begin{proof}
Write the conditional entropy directly in posterior form:
\begin{align*}
H(X_i\mid Y)
&=
-\E\!\left[\log P_{X_i\mid Y}(X_i\mid Y)\right]
\end{align*}
Using (C2) in the matched case $Q_X = P_X$, we compute the derivative
\begin{align*}
\frac{\partial}{\partial t_i} H(X_i\mid Y)
&=
-\E\!\left[
S_{t_i}(X_i,Y_i)\log P_{X_i\mid Y}(X_i\mid Y)
\right]\\
&\quad-
\E\!\left[
\frac{\partial}{\partial t_i}
\log P_{X_i\mid Y}(X_i\mid Y)
\right].
\end{align*}
Using the posterior factorization,
\begin{align*}
\log P_{X_i\mid Y}(x_i\mid y)
&=
\log P_{X_i\mid Y_{-i}}(x_i\mid y_{-i})
+\log W_{t_i}(y_i\mid x_i)
-\log Z_i(y),
\end{align*}
where
\begin{align*}
Z_i(y)\triangleq
\sum_{u\in\mathcal{X}}
P_{X_i\mid Y_{-i}}(u\mid y_{-i})W_{t_i}(y_i\mid u).
\end{align*}
The first two terms are explicit, so the only nontrivial derivative comes from the normalization. Differentiating that normalization term yields
\begin{align*}
\frac{\partial}{\partial t_i}\log Z_i(y)
&=
\E_{X'\sim P_{X_i\mid Y}}
\!\left[S_{t_i}(X',y_i)\mid Y=y\right].
\end{align*}
Substituting this back gives
\begin{align*}
\frac{\partial}{\partial t_i}
\log P_{X_i\mid Y}(x_i\mid y)
&=
S_{t_i}(x_i,y_i)
-\E_{X'\sim P_{X_i\mid Y}}
\!\left[S_{t_i}(X',y_i)\mid Y=y\right].
\end{align*}
Now substitute $x_i=X_i$ and $y=Y$. The remaining expectation vanishes by the tower property, since the conditional expectation of the score term given $Y$ matches the inner term. Hence
\begin{align*}
\frac{\partial}{\partial t_i} H(X\mid Y)
&=
-\E\!\Big[
S_{t_i}(X_i,Y_i)\log P_{X_i\mid Y}(X_i\mid Y)
\Big].
\qedhere
\end{align*}
\end{proof}
\subsection{Proof of Theorem~\ref{thm:gexit-derivative}}
\label{subsec:gexit-derivative-proof}
\begin{proof}
Write
\begin{align*}
\CE(P_{X_i\mid Y},Q_{X_i\mid Y})
&=
-\E\!\left[\log Q_{X_i\mid Y}(X_i\mid Y)\right].
\end{align*}
By (C2), we may differentiate the CE expectation
with respect to the interior coordinate $t_i$, which gives
\begin{align*}
\frac{\partial}{\partial t_i}
\CE(P_{X_i\mid Y},Q_{X_i\mid Y})
&=
-\E\!\left[
S_{t_i}(X_i,Y_i)\log Q_{X_i\mid Y}(X_i\mid Y)
\right]\\
&\quad-
\E\!\left[
\frac{\partial}{\partial t_i}
\log Q_{X_i\mid Y}(X_i\mid Y)
\right],
\end{align*}
where we used the memoryless channel factorization to write
\begin{align*}
\frac{\partial}{\partial t_i} P_{X_i,Y}(x_i,y)
&=
P_{X_i,Y}(x_i,y)\,S_{t_i}(x_i,y_i).
\end{align*}
Since $Q_{X_i\mid Y_{-i}(\mathbf t_{-i})}$ depends on the noise vector only
through $\mathbf t_{-i}$, its $t_i$-derivative is zero.
Thus, we have
\begin{align*}
\log Q_{X_i\mid Y}(x_i\mid y)
&=
\log Q_{X_i\mid Y_{-i}}(x_i\mid y_{-i})
+\log W_{t_i}(y_i\mid x_i)\\
&\quad-
\log \Bigg(
\sum_{u\in\mathcal{X}}
Q_{X_i\mid Y_{-i}}(u\mid y_{-i})W_{t_i}(y_i\mid u)
\Bigg).
\end{align*}
Differentiating the last equation yields
\begin{align*}
\frac{\partial}{\partial t_i}
\log Q_{X_i\mid Y}(x_i\mid y)
&=
S_{t_i}(x_i,y_i)
-
\E_{X'\sim Q_{X_i\mid Y}}
\!\left[S_{t_i}(X',y_i)\mid Y=y\right].
\end{align*}
Substituting back gives
\begin{align*}
\frac{\partial}{\partial t_i}
\CE(P_{X_i\mid Y},Q_{X_i\mid Y})
&=
\E \Big[
-S_{t_i}(X_i,Y_i)\log Q_{X_i\mid Y}(X_i\mid Y)
\Big]\\
&\quad-
\E\!\left[S_{t_i}(X_i,Y_i)\right]\\
&\quad+
\E \Big[
\E_{X'\sim Q_{X_i\mid Y}}
\!\left[S_{t_i}(X',Y_i)\mid Y\right]
\Big].
\end{align*}
The middle term vanishes by the zero-mass derivative in (C1), because
\begin{align*}
\E\!\left[S_{t_i}(X_i,Y_i)\right]
&=
\sum_{x_i} P_{X_i}(x_i)
\int \frac{\partial}{\partial t_i} W_{t_i}(y_i\mid x_i)\,dy_i
= 0.
\end{align*}
This proves \eqref{eq:gexit-derivative}.
\end{proof}

\subsection{Proof of the Mismatched GEXIT Conservation Law}
\label{subsec:gexit-area-proof}
\begin{proof}
Let $P_X, Q_X$ be the distributions on $\mathcal X^n$ satisfying $P_X \ll Q_X$.
Let $P_{X\mid Y(\mathbf t)}$ and $Q_{X\mid Y(\mathbf t)}$ be defined as
\begin{align*}
    P_{X\mid Y(\mathbf t)}(x\mid y)
    &\triangleq
    \frac{P_X(x)\prod_{j=1}^n W_{t_j}(y_j\mid x_j)}{\sum_{\bar x\in\mathcal X^n} P_X(\bar x)\prod_{j=1}^n W_{t_j}(y_j\mid \bar x_j)},\\
    Q_{X\mid Y(\mathbf t)}(x\mid y)
    &\triangleq
    \frac{Q_X(x)\prod_{j=1}^n W_{t_j}(y_j\mid x_j)}{\sum_{\bar x\in\mathcal X^n} Q_X(\bar x)\prod_{j=1}^n W_{t_j}(y_j\mid \bar x_j)}.
\end{align*}
By the chain rule for conditional entropy,
\begin{align*}
\CE\!\big(P_{X\mid Y(\mathbf t)},Q_{X\mid Y(\mathbf t)}\big)
&=
\CE\!\big(P_{X_i\mid Y(\mathbf t)},Q_{X_i\mid Y(\mathbf t)}\big)
+
\CE\!\big(P_{X_{-i}\mid X_i,Y_{-i}(\mathbf t_{-i})},Q_{X_{-i}\mid X_i,Y_{-i}(\mathbf t_{-i})}\big),
\end{align*}
where the second term averages over $(X_i,Y_{-i}(\mathbf t_{-i}))$ under the law induced by $P_X$.  The second term depends only on $\mathbf t_{-i}$, so its derivative with respect to $t_i$ is zero.
Consequently,
\begin{align*}
\frac{\partial}{\partial t_i}
\CE\!\big(P_{X\mid Y(\mathbf{t})},Q_{X\mid Y(\mathbf{t})}\big)
&=
\frac{\partial}{\partial t_i}
\CE\!\big(P_{X_i\mid Y(\mathbf{t})},Q_{X_i\mid Y(\mathbf{t})}\big).
\end{align*}
Let
\begin{align*}
F(\mathbf{t})
\triangleq
\CE\!\big(P_{X\mid Y(\mathbf{t})},Q_{X\mid Y(\mathbf{t})}\big).
\end{align*}

For every compact subset $K \subset (0,1)$, assume the envelope bounds in (C2) hold uniformly for $t$ in a compact neighborhood of $\gamma(K)$, with $\sum_i |\gamma_i ' (t)| B_i$ integrable.
Then $F\circ \gamma$ is absolutely continuous and the chain rule holds almost everywhere.
Thus, integrating along
$\gamma$ gives
\begin{align*}
\int_0^1\sum_{i=1}^n
\frac{\partial}{\partial t_i}
F(\mathbf{t})\big|_{\mathbf{t}=\gamma(\tau)}\gamma_i^\prime(\tau)\,d\tau
=
F(\gamma(1))-F(\gamma(0)).
\end{align*}
Using the endpoint convention (C3), at the clean endpoint, $Y(\gamma(0))=X$ under both input laws. Hence
$P_{X\mid Y(\gamma(0))=y}=Q_{X\mid Y(\gamma(0))=y}=\delta_y$, and
$F(\gamma(0))=0$. At the uninformative endpoint, the channel output is
independent of $X$ under both input laws. Hence
$P_{X\mid Y(\gamma(1))}=P_X$ and $Q_{X\mid Y(\gamma(1))}=Q_X$,
so $F(\gamma(1))=\CE(P_X,Q_X)$.
\end{proof}

\subsection{Proof of Proposition~\ref{thm:exit-derivative}}
\label{subsec:masked-exit-proof}
\begin{proof}
Let $m=\texttt{[MASK]}$. By Theorem~\ref{thm:gexit-derivative},
\begin{align*}
\frac{\partial}{\partial t_i}
\CE(P_{X_i\mid Y},Q_{X_i\mid Y})
&=
\E\!\left[
-S_{t_i}(X_i,Y_i)\log Q_{X_i\mid Y}(X_i\mid Y)
\right]\\
&\quad+
\E\!\left[
\E_{X'\sim Q_{X_i\mid Y}}
\!\left[S_{t_i}(X',Y_i)\mid Y\right]
\right].
\end{align*}
Condition on $Y_{-i}$. Under the masking channel, $Y_i=m$ with probability $t_i$ and $Y_i=X_i$ with probability $1-t_i$. Since $S_{t_i}(x,m)=1/t_i$ and $S_{t_i}(x,x)=-1/(1-t_i)$, the first term splits as
\begin{align*}
\E\!\left[
-S_{t_i}(X_i,Y_i)\log Q_{X_i\mid Y}(X_i\mid Y)
\right]
&=
\CE_i^{\mathrm{mask}}(\mathbf{t})-\CE_i^{\mathrm{obs}}(\mathbf{t}).
\end{align*}
The observed branch enters with a minus sign because $S_{t_i}(x,x)=-1/(1-t_i)$.
For the correction term, when $Y_i=m$ the score average is $1/t_i$, and when $Y_i=X_i$ the score average is $-1/(1-t_i)$. Therefore
\begin{align*}
\E\!\left[
\E_{X'\sim Q_{X_i\mid Y}}
\!\left[S_{t_i}(X',Y_i)\mid Y\right]
\right]
&=
\frac{1}{t_i}\Pr(Y_i=m)-\frac{1}{1-t_i}\Pr(Y_i=X_i)\\
&=1-1=0.
\end{align*}
Hence
\begin{align*}
\frac{\partial}{\partial t_i}
\CE(P_{X_i\mid Y},Q_{X_i\mid Y})
&=
\CE_i^{\mathrm{mask}}(\mathbf{t})-\CE_i^{\mathrm{obs}}(\mathbf{t}),
\end{align*}
which is exactly \eqref{eq:exit-derivative}.
\end{proof}

\subsection{Proof of Proposition~\ref{thm:qary-gexit-derivative} (uniform replacement channel)}
\label{subsec:qary-gexit-proof}
\begin{proof}
It is convenient to first differentiate with respect to the native
pass-through probability $a_i$. For the uniform replacement channel,
\begin{align*}
S_{a_i}(x,y)
&=
\begin{cases}
1/a_i, & y=x,\\
-1/(1-a_i), & y\neq x.
\end{cases}
\end{align*}
Applying the score-form calculation from Theorem~\ref{thm:gexit-derivative}, the channel-score-weighted NLL
term becomes
\begin{align*}
\E\!\left[
-S_{a_i}(X_i,Y_i)\log Q_{X_i\mid Y}(X_i\mid Y)
\right]
&=
\CE^{\mathrm{cor}}_i(\mathbf{s})-\CE^{\mathrm{err}}_i(\mathbf{s}),
\end{align*}
because $\Pr(Y_i=X_i)=a_i$ and $\Pr(Y_i\neq X_i)=1-a_i$. For the correction
term, condition on $Y$. Then
\begin{align*}
\E_{X'\sim Q_{X_i\mid Y}}
\!\left[S_{a_i}(X',Y_i)\mid Y\right]
&=
\frac{Q_{X_i\mid Y}(Y_i\mid Y)}{a_i}
-\frac{1-Q_{X_i\mid Y}(Y_i\mid Y)}{1-a_i}\\
&=
\frac{Q_{X_i\mid Y}(Y_i\mid Y)}{a_i(1-a_i)}
-\frac{1}{1-a_i}.
\end{align*}
Taking expectation over $Y$ and substituting into the same score-form
calculation gives the native-parameter derivative
\begin{align*}
\frac{\partial}{\partial a_i}
\CE\big(P_{X_i\mid Y},Q_{X_i\mid Y}\big)
&=
\CE^{\mathrm{cor}}_i(\mathbf{s})-\CE^{\mathrm{err}}_i(\mathbf{s})
+\frac{1}{a_i(1-a_i)}
\E\!\left[
Q_{X_i\mid Y}(Y_i\mid Y)
\right]
-\frac{1}{1-a_i}.
\end{align*}
Since $a_i=a(s_i)=1-s_i(1-1/q)$, we have
$da_i/ds_i=-(1-1/q)$. Applying the chain rule proves
\eqref{eq:qary-gexit-derivative}.
\end{proof}

\subsection{Proof of Proposition~\ref{cor:gaussian-ce-derivative} from the marginal GEXIT formula}
\label{subsec:gaussian-ce-proof}

The Gaussian CE area formula is a direct consequence of the mismatched GEXIT
formula in Theorem~\ref{thm:gexit-derivative}.
Assume $P_X \ll Q_X$, let $U_i=f(X_i)\in\mathbb{R}^d$, and
$U=(U_1,\ldots,U_n)\in\mathbb{R}^{nd}$, where $f$ is the injective embedding
used in Proposition~\ref{cor:gaussian-ce-derivative}. Instead of
$Y_i(\rho)=U_i+\rho^{-1/2}Z_i$, it is convenient to use the equivalent
observation model
$$
R_i(\rho)=\sqrt{\rho}\,U_i+Z_i,
\qquad
Z_i\sim\mathcal N(0,I_d),
$$
and write $R_\rho=(R_1(\rho),\ldots,R_n(\rho))$. Since
$R_\rho=\sqrt{\rho}\,Y(\rho)$, the two channels induce the same posterior
marginals.

For a nonuniform SNR vector
$\boldsymbol\rho=(\rho_1,\ldots,\rho_n)$, write
$R(\boldsymbol\rho)$ for the coordinatewise Gaussian observation. The SNR
parameter runs in the opposite direction from the degradation coordinate in
\eqref{eq:gexit-area-general}: $\rho=0$ is uninformative, while
$\rho=\infty$ is clean. Therefore, applying
\eqref{eq:gexit-area-general} to the uniform SNR path gives
\begin{align*}
\CE(P_X,Q_X)
&=
-
\int_0^\infty
\sum_{i=1}^n
D_i(\rho)\,d\rho,
\end{align*}
where
\begin{align*}
D_i(\rho)
&\triangleq
\left.
\frac{\partial}{\partial \rho_i}
\CE\!\left(
P_{X_i\mid R(\boldsymbol\rho)},
Q_{X_i\mid R(\boldsymbol\rho)}
\right)
\right|_{\boldsymbol\rho=(\rho,\ldots,\rho)}.
\end{align*}

The Gaussian channel score with respect to $\rho$ is
$$
S_\rho(u_i,r_i)
=
\frac{\partial}{\partial\rho}
\log W_\rho(r_i|u_i)
=
\frac{1}{2\sqrt{\rho}}
(r_i-\sqrt{\rho}u_i)^\top u_i.
$$
For each coordinate, define the model marginal posterior mean
$$
m_{Q,i}(r)
\triangleq
\E_{X'\sim Q_{X_i\mid R_\rho}(\cdot\mid r)}
\!\left[f(X')\right],
$$
and set $m_Q(r)=(m_{Q,1}(r),\ldots,m_{Q,n}(r))$.
Theorem~\ref{thm:gexit-derivative} gives the marginal score decomposition
\begin{align*}
D_i(\rho)
&=
\underbrace{
\E_P
\!\left[
-
S_\rho(U_i,R_i)
\log Q_{X_i\mid R_\rho}(X_i\mid R_\rho)
\right]
}_{\substack{T_{1,i}\\\text{log-likelihood term}}}
\\
&+
\underbrace{
\E_P
\!\left[
\E_{X'\sim Q_{X_i\mid R_\rho}(\cdot\mid R_\rho)}
\!\left[
S_\rho(f(X'),R_i)\mid R_\rho
\right]
\right]
}_{\substack{T_{2,i}\\\text{posterior-score correction}}}.
\end{align*}

To convert the local Gaussian noise factors in $T_{1,i}$ and $T_{2,i}$ into
derivatives with respect to $r_i$, we use the following vector Stein identity.
\begin{lemma}[Vector Stein identity]
\label{lem:vector-stein}
Let $Z\sim\mathcal N(0,I_D)$ and $R=\mu+Z$, where $\mu\in\mathbb R^D$ is
fixed. Let $\phi_\mu$ denote the density of $\mathcal N(\mu,I_D)$, and suppose
the expectations below are finite.

If $h\colon\mathbb R^D\to\mathbb R$ is continuously differentiable and, for
each coordinate $j$, $h(r)\phi_\mu(r)$ vanishes as $r_j\to\pm\infty$ with the
other coordinates fixed, then, for every fixed $a\in\mathbb R^D$,
\begin{equation}
\label{eq:vector-stein-scalar}
\E\!\left[Z^\top a\,h(R)\right]
=
\E\!\left[a^\top\nabla h(R)\right].
\end{equation}
If $g\colon\mathbb R^D\to\mathbb R^D$ is continuously differentiable and, for
each coordinate $j$, $g_j(r)\phi_\mu(r)$ vanishes as $r_j\to\pm\infty$ with the
other coordinates fixed, then
\begin{equation}
\label{eq:vector-stein-divergence}
\E\!\left[Z^\top g(R)\right]
=
\E\!\left[\nabla\!\cdot g(R)\right].
\end{equation}
\end{lemma}

Once Stein's identity has moved derivatives onto the marginal $Q$-posterior,
the next identities evaluate the resulting local gradient and divergence.
\begin{lemma}[Gaussian marginal-posterior gradient and divergence]
\label{lem:gaussian-posterior-grad-div}
For the coordinatewise Gaussian channel
$R_i(\rho)=\sqrt{\rho}\,U_i+Z_i$ under the model law $Q$, let
$m_{Q,i}(r)=\E_Q[U_i|R_\rho=r]$. Then, for every $x_i$ with positive marginal
posterior support,
\begin{equation}
\label{eq:gaussian-posterior-score}
\nabla_{r_i}
\log Q_{X_i\mid R_\rho}(x_i|r)
=
\sqrt{\rho}
\bigl(
f(x_i)-m_{Q,i}(r)
\bigr),
\end{equation}
and
\begin{equation}
\label{eq:gaussian-posterior-divergence}
\nabla_{r_i}\!\cdot m_{Q,i}(r)
=
\sqrt{\rho}
\left(
\E_{X'\sim Q_{X_i\mid R_\rho}(\cdot\mid r)}
\!\left[\|f(X')\|^2\right]
-
\|m_{Q,i}(r)\|^2
\right).
\end{equation}
\end{lemma}

Lemma~\ref{lem:vector-stein} is the standard finite-dimensional Gaussian
integration-by-parts identity; see, e.g.,
\cite[Thm.~5.1.8, p.~209]{Bogachev1998GaussianMeasures}. The proof of
Lemma~\ref{lem:gaussian-posterior-grad-div} is given at the end of this
subsection.
\paragraph{Marginal log-likelihood term.}
Let $T_{1,i}$ denote the first summand in the marginal score decomposition.
\begin{align*}
T_{1,i}
&\triangleq
\E_P
\!\left[
-
S_\rho(U_i,R_i)
\log Q_{X_i\mid R_\rho}(X_i\mid R_\rho)
\right]\\
&\stackrel{(a)}{=}
-\frac{1}{2\sqrt{\rho}}
\E_P
\!\left[
Z_i^\top U_i\log Q_{X_i\mid R_\rho}(X_i\mid R_\rho)
\right]\\
&\stackrel{(b)}{=}
-\frac{1}{2\sqrt{\rho}}
\E_P
\!\left[
\E_P
\!\left[
Z_i^\top U_i\log Q_{X_i\mid R_\rho}(X_i\mid R_\rho)
\,\middle|\,
X,R_{-i}
\right]
\right]\\
&\stackrel{(c)}{=}
-\frac{1}{2\sqrt{\rho}}
\E_P
\!\left[
U_i^\top
\nabla_{r_i}\log Q_{X_i\mid R_\rho}(X_i\mid R_\rho)
\right]\\
&\stackrel{(d)}{=}
-\frac12
\E_P
\!\left[
U_i^\top\bigl(U_i-m_{Q,i}(R_\rho)\bigr)
\right]\\
&=
-\frac12
\E_P
\left[
\|U_i\|^2
-
U_i^\top m_{Q,i}(R_\rho)
\right].
\end{align*}
Here, (a) substitutes the Gaussian score and uses
$R_i-\sqrt{\rho}U_i=Z_i$. Step (b) is the tower property, conditioning on
$X$ and the side information $R_{-i}$. Under this conditioning, $U_i$ and
$R_{-i}$ are fixed, while $R_i=\sqrt{\rho}U_i+Z_i$ is the only remaining
Gaussian variable. Step (c) applies the scalar identity
\eqref{eq:vector-stein-scalar} to the function
$h_{x_i,r_{-i}}(r_i)=\log Q_{X_i\mid R_\rho}(x_i\mid r_i,r_{-i})$.
Step (d) uses the marginal posterior-score identity
\eqref{eq:gaussian-posterior-score}; the final equality only expands the inner
product.

\paragraph{Marginal posterior-score correction term.}
Let $T_{2,i}$ denote the second summand in the marginal score decomposition.
\begin{align*}
T_{2,i}
&\triangleq
\E_P
\!\left[
\E_{X'\sim Q_{X_i\mid R_\rho}(\cdot\mid R_\rho)}
\!\left[
S_\rho(f(X'),R_i)
\,\middle|\,
R_\rho
\right]
\right]\\
&=
\frac{1}{2\sqrt{\rho}}
\E_P
\left[
R_i^\top m_{Q,i}(R_\rho)
-
\sqrt{\rho}\,
\E_{X'\sim Q_{X_i\mid R_\rho}(\cdot\mid R_\rho)}
\!\left[\|f(X')\|^2\right]
\right].
\end{align*}
The equality holds by taking the marginal $Q$-posterior expectation of the
score: the term linear in $f(X')$ gives $m_{Q,i}(R_\rho)$, and the quadratic
term gives the marginal posterior second moment. The remaining term is
evaluated as
\begin{align*}
\E_P\!\left[R_i^\top m_{Q,i}(R_\rho)\right]
&\stackrel{(a)}{=}
\sqrt{\rho}\,
\E_P\!\left[U_i^\top m_{Q,i}(R_\rho)\right]
+
\E_P\!\left[Z_i^\top m_{Q,i}(R_\rho)\right]\\
&\stackrel{(b)}{=}
\sqrt{\rho}\,
\E_P\!\left[U_i^\top m_{Q,i}(R_\rho)\right]
+
\E_P\!\left[\nabla_{r_i}\!\cdot m_{Q,i}(R_\rho)\right]\\
&\stackrel{(c)}{=}
\sqrt{\rho}\,
\E_P
\!\left[
U_i^\top m_{Q,i}(R_\rho)
+
\E_{X'\sim Q_{X_i\mid R_\rho}(\cdot\mid R_\rho)}
\!\left[\|f(X')\|^2\right]
-
\|m_{Q,i}(R_\rho)\|^2
\right].
\end{align*}
Here, (a) expands $R_i=\sqrt{\rho}U_i+Z_i$. Step (b) applies the vector-field
identity \eqref{eq:vector-stein-divergence} conditionally on $X$ and
$R_{-i}$, with $g(r_i)=m_{Q,i}(r_i,R_{-i})$, and then averages over
$(X,R_{-i})$. Step (c) uses the marginal divergence identity
\eqref{eq:gaussian-posterior-divergence}.

\paragraph{Tail conditions for Lemma~\ref{lem:vector-stein}.}
It remains to verify the boundary conditions in the two applications of the Stein identity above.
Since $\mathcal X^n$ is finite and $f$ is fixed,
$B\triangleq\max_x\|f(x)\|$ is finite.  For every $x_i$ inside the
$P$-expectation, finite cross entropy implies positive model marginal
posterior probability.  Fixing $r_{-i}$, write $u_{x_i}=f(x_i)$.  The marginal
posterior log-probability used in Step (c) has the form
\begin{align*}
\log Q_{X_i\mid R_\rho}(x_i|r_i,r_{-i})
&=
c_{x_i}(r_{-i})+\sqrt{\rho}\,r_i^\top u_{x_i}
-
\log
\sum_{x_i'}
\exp\!\left\{
c_{x_i'}(r_{-i})+\sqrt{\rho}\,r_i^\top u_{x_i'}
\right\},
\end{align*}
where
$c_{x_i}(r_{-i})=\log Q_{X_i\mid R_{-i}}(x_i|r_{-i})
-\frac{\rho}{2}\|u_{x_i}\|^2$. The last term is a finite log-sum-exp of affine
functions in $r_i$, so
$|\log Q_{X_i\mid R_\rho}(x_i|r_i,r_{-i})|\le C(r_{-i})(1+\|r_i\|)$ for a
finite constant $C(r_{-i})$. Also,
$\|\nabla_{r_i}\log Q_{X_i\mid R_\rho}(x_i|r)\|\le 2\sqrt{\rho}\,B$,
$\|m_{Q,i}(r)\|\le B$, and
$|\nabla_{r_i}\!\cdot m_{Q,i}(r)|\le \sqrt{\rho}\,B^2$. Therefore the products
of the scalar test function in Step (c) and the vector-field components in
Step (b) with the Gaussian density in $r_i$ vanish as any local coordinate goes
to $\pm\infty$.

Substituting the preceding identity
into the expression for $T_{2,i}$ cancels the marginal posterior second-moment
terms and gives
\begin{align*}
T_{2,i}
&=
\frac12
\E_P
\left[
U_i^\top m_{Q,i}(R_\rho)
-
\|m_{Q,i}(R_\rho)\|^2
\right].
\end{align*}
\paragraph{Combining and integrating.}
Adding the two contributions gives
$$
D_i(\rho)
=
-\frac12
\E_P
\left[
\|U_i-m_{Q,i}(R_\rho)\|^2
\right].
$$
This equality follows by adding the evaluated $T_{1,i}$ and $T_{2,i}$ terms:
the cross terms combine to $2U_i^\top m_{Q,i}(R_\rho)$, completing the square.

Define the mismatched estimation error
$$
\operatorname{MSE}_Q(\rho)
\triangleq
\E_P
\left[
\sum_{i=1}^n
\|U_i-m_{Q,i}(R_\rho)\|^2
\right].
$$
Plugging the marginal derivative into the conservation law yields
\begin{align*}
\CE(P_X,Q_X)
&=
\frac12
\int_0^\infty
\operatorname{MSE}_Q(\rho)\,d\rho,
\end{align*}
which is exactly Proposition~\ref{cor:gaussian-ce-derivative}, because
$R_\rho=\sqrt{\rho}\,Y(\rho)$ induces the same marginal posteriors as
$Y(\rho)$ and $\tilde Q_{X|Y(\rho)}$ is the product of those marginals.
In particular, the integrand depends on $Q$ only through the coordinatewise posterior means.
Thus, the same integrand can be evaluated from the collection of marginal posteriors, although exact equality requires that these marginals arise from a globally consistent $Q_X$.

\begin{proof}[Proof of Lemma~\ref{lem:gaussian-posterior-grad-div}]
Write $u_{x_i}=f(x_i)$ and fix $r_{-i}$. By the base posterior factorization
\eqref{eq:base-posterior-factorization}, the local normalization is
$$
Z_{Q,i}(r)
\triangleq
\sum_{x_i'} Q_{X_i\mid R_{-i}}(x_i'|r_{-i})
\exp\left\{
-\frac12\|r_i-\sqrt{\rho}u_{x_i'}\|^2
\right\}.
$$
For $x_i$ with positive marginal posterior support, Bayes' rule gives
$$
\log Q_{X_i\mid R_\rho}(x_i|r)
=
\log Q_{X_i\mid R_{-i}}(x_i|r_{-i})
-\frac12\|r_i-\sqrt{\rho}u_{x_i}\|^2
-\log Z_{Q,i}(r).
$$
Differentiating this identity with respect to $r_i$ gives
$$
\nabla_{r_i}\log Q_{X_i\mid R_\rho}(x_i|r)
=
-(r_i-\sqrt{\rho}u_{x_i})-\nabla_{r_i}\log Z_{Q,i}(r).
$$
The normalization term satisfies
\begin{align*}
\nabla_{r_i}\log Z_{Q,i}(r)
&=
\sum_{x_i'} Q_{X_i\mid R_\rho}(x_i'|r)
\bigl(-(r_i-\sqrt{\rho}u_{x_i'})\bigr)\\
&=
-r_i+\sqrt{\rho}\,m_{Q,i}(r).
\end{align*}
Substituting this expression into the preceding equation gives
$$
\nabla_{r_i}\log Q_{X_i\mid R_\rho}(x_i|r)
=
\sqrt{\rho}\bigl(u_{x_i}-m_{Q,i}(r)\bigr),
$$
which proves the gradient identity. For the divergence identity, use
$m_{Q,i}(r)=\sum_{x_i}Q_{X_i\mid R_\rho}(x_i|r)u_{x_i}$ and differentiate
componentwise in the local vector $r_i$:
\begin{align*}
\nabla_{r_i}\!\cdot m_{Q,i}(r)
&\stackrel{(a)}{=}
\sum_{j=1}^d
\sum_{x_i} u_{x_i,j}
\frac{\partial}{\partial r_{i,j}}Q_{X_i\mid R_\rho}(x_i|r)\\
&\stackrel{(b)}{=}
\sum_{j=1}^d
\sum_{x_i} u_{x_i,j}Q_{X_i\mid R_\rho}(x_i|r)
\frac{\partial}{\partial r_{i,j}}
\log Q_{X_i\mid R_\rho}(x_i|r)\\
&\stackrel{(c)}{=}
\sqrt{\rho}
\sum_{j=1}^d
\left(
\E_{X'\sim Q_{X_i\mid R_\rho}(\cdot\mid r)}
\!\left[f_j(X')^2\right]
-
m_{Q,i,j}(r)^2
\right)\\
&\stackrel{(d)}{=}
\sqrt{\rho}
\left(
\E_{X'\sim Q_{X_i\mid R_\rho}(\cdot\mid r)}
\!\left[\|f(X')\|^2\right]
-
\|m_{Q,i}(r)\|^2
\right).
\end{align*}
Here (a) expands the divergence of
$m_{Q,i}(r)=\sum_{x_i}Q_{X_i\mid R_\rho}(x_i|r)u_{x_i}$ coordinate by
coordinate; the vectors $u_{x_i}$ are fixed, so the derivative acts only on
$Q_{X_i\mid R_\rho}(x_i|r)$. Step (b) uses
$\partial_{r_{i,j}}Q_{X_i\mid R_\rho}
=Q_{X_i\mid R_\rho}\partial_{r_{i,j}}\log Q_{X_i\mid R_\rho}$. Step (c)
substitutes the gradient identity
$\partial_{r_{i,j}}\log Q_{X_i\mid R_\rho}(x_i|r)
=\sqrt{\rho}(u_{x_i,j}-m_{Q,i,j}(r))$ and then collects the marginal posterior
average
$$
\sum_{x_i} Q_{X_i\mid R_\rho}(x_i|r)
u_{x_i,j}(u_{x_i,j}-m_{Q,i,j}(r))
=
\E_{X'\sim Q_{X_i\mid R_\rho}(\cdot\mid r)}
\!\left[f_j(X')^2\right]
-
m_{Q,i,j}(r)^2.
$$
Step (d) sums these local coordinate variances into the second form in
\eqref{eq:gaussian-posterior-divergence}.
\end{proof}

\subsection{Proof via I--MMSE and Verd\'u's mismatched estimation formula}
\label{subsec:gaussian-ce-verdu-proof}
\begin{proof}
Use the notation from Subsection~\ref{subsec:gaussian-ce-proof}:
$U_i=f(X_i)$, $U=(U_1,\ldots,U_n)$, and
$R_\rho=\sqrt{\rho}\,U+Z$. Let
$m_P(r)=\E_P[U|R_\rho=r]$ and let
$m_Q(r)=(m_{Q,1}(r),\ldots,m_{Q,n}(r))$ be the stacked vector of model
marginal posterior means. Define
$$
\operatorname{MMSE}_P(\rho)
\triangleq
\E_P\!\left[\|U-m_P(R_\rho)\|^2\right],
\qquad
\operatorname{MSE}_Q(\rho)
\triangleq
\E_P\!\left[\|U-m_Q(R_\rho)\|^2\right].
$$
With this SNR parameterization, and because $X$ has finite alphabet and $f$ is
injective, the Gaussian entropy--minimum mean-square error (MMSE) relation gives
\begin{align*}
H(P_X)
&=
\frac{1}{2}\int_0^\infty
\operatorname{MMSE}_P(\rho)\,d\rho.
\end{align*}
Verd\'u's mismatched relative-entropy formula, written with the same posterior
means, gives
\begin{align*}
D\!\big(P_X\|Q_X\big)
&=
\frac{1}{2}\int_0^\infty
\Big(
\operatorname{MSE}_Q(\rho)
-
\operatorname{MMSE}_P(\rho)
\Big)\,d\rho,
\end{align*}
Combining the two identities yields the cross-entropy area formula:
\begin{align*}
\CE\!\big(P_X,Q_X\big)
&=
\frac{1}{2}\int_0^\infty
\operatorname{MSE}_Q(\rho)\,d\rho.
\end{align*}
\end{proof}

\subsection{Connection between the GEXIT derivation and Verd\'u's formula}
\label{subsec:gaussian-ce-verdu-connection}

The GEXIT proof also explains how the cross-entropy identity decomposes into
the matched I--MMSE identity and Verd\'u's mismatched relative-entropy formula.

Throughout this subsection we keep the notation of
Subsection~\ref{subsec:gaussian-ce-proof}: $U_i=f(X_i)$,
$U=(U_1,\ldots,U_n)$, $R_\rho=\sqrt{\rho}\,U+Z$,
$m_P(r)=\E_P[U|R_\rho=r]$, and
$m_Q(r)=(m_{Q,1}(r),\ldots,m_{Q,n}(r))$. Thus
$$
\operatorname{MMSE}_P(\rho)
\triangleq
\E_P\!\left[\|U-m_P(R_\rho)\|^2\right],
\qquad
\operatorname{MSE}_Q(\rho)
\triangleq
\E_P\!\left[\|U-m_Q(R_\rho)\|^2\right].
$$

\paragraph{Matched model.}
Setting $Q=P$ in the GEXIT derivation gives
\begin{equation}
F_P(\rho)
=
H(X|R_\rho),
\end{equation}
and the same calculation yields
\begin{equation}
-
\frac{d}{d\rho}
H(X|R_\rho)
=
\frac12
\operatorname{MMSE}_P(\rho).
\end{equation}
Consequently,
\begin{equation}
H(P_X)
=
\frac12
\int_0^\infty
\operatorname{MMSE}_P(\rho)\,d\rho,
\end{equation}
which is the classical I--MMSE identity.

\paragraph{Verd\'u's mismatched relative-entropy identity.}
Subtracting the matched identity from the mismatched CE identity gives
\begin{equation}
\begin{aligned}
D(P_X\|Q_X)
&=
\CE(P_X,Q_X)-H(P_X)\\
&=
\frac12
\int_0^\infty
\left[
\operatorname{MSE}_Q(\rho)
-
\operatorname{MMSE}_P(\rho)
\right]
d\rho,
\end{aligned}
\end{equation}
which is Verd\'u's infinite-SNR mismatched estimation formula.

Moreover, by the chain rule for relative entropy through the Gaussian channel,
\begin{equation}
D(P_X\|Q_X)
=
D(P_{R_\rho}\|Q_{R_\rho})
+
\E_{P_{R_\rho}}
D(P_{X|R_\rho}\|Q_{X|R_\rho}).
\end{equation}
Since
\begin{equation}
\E_{P_{R_\rho}}
D(P_{X|R_\rho}\|Q_{X|R_\rho})
=
F_Q(\rho)-F_P(\rho),
\end{equation}
we obtain
\begin{equation}
\frac{d}{d\rho}
D(P_{R_\rho}\|Q_{R_\rho})
=
-
\frac{d}{d\rho}
\bigl(
F_Q(\rho)-F_P(\rho)
\bigr)
=
\frac12
\left[
\operatorname{MSE}_Q(\rho)
-
\operatorname{MMSE}_P(\rho)
\right].
\end{equation}
Since $R_0=Z$ is independent of $X$, we have
\begin{equation}
P_{R_0}=Q_{R_0},
\end{equation}
and therefore
\begin{equation}
D(P_{R_\rho}\|Q_{R_\rho})
=
\frac12
\int_0^\rho
\left[
\operatorname{MSE}_Q(\gamma)
-
\operatorname{MMSE}_P(\gamma)
\right]
d\gamma.
\end{equation}

This is Verd\'u's finite-SNR output-divergence identity. Letting
$\rho\rightarrow\infty$ recovers the previous relative-entropy formula.

%% file: sections/appendix_repro.tex
\subsection{Integration Endpoints}
\label{app:endpoints}
For the conservation laws to recover cross-entropy, the integration path must
connect a clean-data endpoint to a pure-noise endpoint, where the observation
is independent of the clean token. Each channel family in the paper is
parameterized to satisfy this condition.

For masking, the corruption parameter satisfies $t=0$ at clean data and $t=1$ at
pure noise. For the uniform channel, degradation progress $s$ runs from
$s=0$ at the clean endpoint to $s=1$ at pure noise; equivalently, the
pass-through probability is $a(s)=1-s(1-1/|\mathcal X|)$, ranging from
$a=1$ to $a=1/|\mathcal X|$. For Gaussian diffusion, SNR is a native parameter
whose direction is reversed relative to degradation progress: the low-SNR limit
is pure noise and the high-SNR limit is pure signal. In practice, we choose the
high-SNR endpoint so that the
posterior-mean MSE induced by the model is negligible on the evaluation grid.
These endpoint constraints are what make the integrated local derivative equal
to the desired likelihood-scale quantity.

\subsection{Synthetic Markov Experiments}\label{app:markov}
\paragraph{Experimental setting.}
The synthetic experiments use first-order Markov sources with alphabet sizes
$K\in\{8,16,32,64,128,256,512\}$ and sequence length 256. For each value of
$K$, every transition row is sampled from a symmetric Dirichlet distribution
with concentration $\alpha=0.05$, and the initial distribution of every
sequence is set to the stationary distribution of the sampled chain. All
objectives use the same transformer architecture: embedding dimension 256, 8
layers, 8 attention heads, and dropout 0.1. Models are trained for $100,000$
optimization steps with batch size 256. Optimization uses Adam \citep{KingmaBa2015Adam}
with learning rate
$5\times10^{-4}$. Training and evaluation use the same integration endpoints
for each channel family, so the conservation-law area is computed over the
same path on which the denoiser is trained. For the Gaussian objective,
the SNR interval is $[10^{-7},10^{-2}]$; this range was selected after verifying
that the MSE vanishes for larger SNRs.

\paragraph{Conservation-law curve comparison.}
To compare the three corruption families under controlled conditions, we
evaluate masked, uniform, and Gaussian conservation-law curves on the same
$K=128$ Markov source. Figure~\ref{fig:exp_markov_channel_comparison_curves} plots the learned
integrands against exact references computed by Bahl--Cocke--Jelinek--Raviv (BCJR) forward--backward inference
\citep{Bahl1974BCJR}. The comparison isolates the numerical effect of the
channel parameterization: all three panels correspond to the same target
entropy, but they induce different local denoising problems and therefore
different finite-model approximation errors.

\begin{figure}[h]
\centering
\input{shared/figures/tikz/exp_markov_channel_comparison_curves}
\caption{Comparison of learned and exact conservation-law curves for a $K=128$ Markov source under three channel parameterizations. Each panel reports the neural estimate together with the BCJR reference obtained by exact forward--backward inference: masked noise as a function of mask rate $t$, uniform noise as a function of replace rate $1-t$, and Gaussian noise as a function of SNR $\rho$. The Gaussian panel is plotted on linear horizontal and vertical scales.}
\label{fig:exp_markov_channel_comparison_curves}
\end{figure}
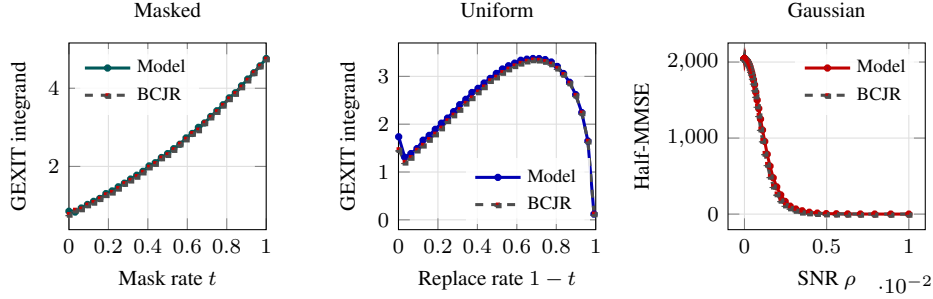

\subsection{Language Modeling on \texttt{text8}}\label{app:lm-repro}
\paragraph{Data preparation.}
The language-model experiments are conducted on the standard \texttt{text8}
corpus from the Matt Mahoney benchmark distribution. The dataset preparation
script extracts the raw 100M-character stream from \texttt{text8.zip} and forms
deterministic 90M/5M/5M train/validation/test splits. The raw byte stream is
used without additional normalization or token filtering. Consequently, the
tokenizer is the identity byte tokenizer with vocabulary size 256, and each
token corresponds to a single byte. Training examples are sampled as independent
contiguous crops of length 256 from the training split, whereas evaluation is
performed on non-overlapping blocks from the validation or test split.

\paragraph{Model.}
All \texttt{text8} diffusion objectives use a common transformer backbone with
hidden size 768, 12 layers, 12 attention heads, feed-forward width 3072, and
dropout 0.1. The models for masked, uniform, and Gaussian noise are parameterized by a
bidirectional discrete-token transformer; the masked model additionally augments
the input vocabulary with a mask symbol. The Gaussian-noise model first embeds
the discrete values in the embedding space and adds noise to the embedding
vectors. All diffusion models use diffusion transformer (DiT)-style timestep conditioning \citep{PeeblesXie2023DiT} with a 256-dimensional sinusoidal input embedding and
a 768-dimensional conditioning width. This yields 128.6M parameters for the
discrete-channel models and 129.2M parameters for the
Gaussian model.

\paragraph{Training and evaluation.}
All reported \texttt{text8} runs use seed 2026 and are optimized with AdamW \citep{LoshchilovHutter2019AdamW}
using $\beta_1=0.9$, $\beta_2=0.95$, weight decay 0.1, gradient clipping at norm
1.0, a 10k-step learning-rate warmup, and minimum learning rate $10^{-5}$. Each
run is trained for 500K optimization steps. The masked and uniform
noise models use learning rate $3\times10^{-4}$, while the Gaussian-noise model uses learning rate
$2\times10^{-4}$. The effective batch size is 512 sequences. Mask noise is sampled
uniformly from $[0,1]$ and evaluated on 100 mask points. The uniform-noise model samples the
channel parameter over $[1/256,0.9999]$ and is evaluated on 100 points. For Gaussian noise, the SNR is sampled
log-uniformly over $[10^{-6},1]$ and evaluated on 100 SNR points. Test curves are
computed on the entire test set for each point across the noise-parameter range.

\paragraph{Resources.}
Each \texttt{text8} simulation used a single NVIDIA H200 graphics processing unit (GPU). Dataset
preparation is central processing unit (CPU)-only and is performed once before training. Each run records
the resolved configuration, training history, checkpoints, final metrics, and
exported evaluation curves.

\subsection{CIFAR-10 Image Modeling}\label{app:cifar-repro}
\paragraph{Data preparation.}
The image experiments use CIFAR-10 \citep{Krizhevsky2009CIFAR10}. Each image is
converted to RGB and represented at the subpixel level as 8-bit tokens, yielding
sequences of length $32\cdot32\cdot3=3072$ over a 256-symbol alphabet. We reserve
the final 2\% of the training split for validation, resulting in 49,000 training
images, 1,000 validation images, and 10,000 test images. Training uses random
horizontal flips with probability 0.5, whereas evaluation is performed on the
fixed validation and test token tensors.

\paragraph{Model.}
All CIFAR-10 objectives use a DDPM-style U-Net backbone \citep{Ho2020DDPM} with
128 base channels, channel multipliers $(1,2,2,2)$, two residual blocks per
resolution, self-attention at spatial resolution 16, dropout 0.1, a
128-dimensional sinusoidal timestep input embedding, and a 512-dimensional
timestep embedding width. The masked and uniform objectives use the discrete
RGB-token version of this architecture, while Gaussian uses the
corresponding continuous RGB-token variant.

\paragraph{Training and evaluation.}
All runs use seed~2026 and AdamW \citep{LoshchilovHutter2019AdamW} with
$\beta_1=0.9$, $\beta_2=0.95$, zero weight decay, gradient clipping at norm 1.0,
a 1000-step learning-rate warmup, and minimum learning rate $10^{-5}$. Each model
is trained for 300K optimization steps. The masked and uniform objectives use learning rate
$3\times10^{-4}$ with batch size 128. Gaussian uses
learning rate $2\times10^{-4}$, batch size 128, and log-uniform SNR sampling over
$[10^{-6},10]$. Mask noise is sampled
uniformly from $[0,1]$ and evaluated on 100 mask points. The uniform-noise model samples the
channel parameter over $[1/256,0.9999]$ and is evaluated on 100 points. For Gaussian noise, the SNR is sampled
log-uniformly over $[10^{-6},1]$ and evaluated on 100 SNR points. Test curves are
computed on the entire test set for each point across the noise-parameter range.

\paragraph{Noise schedules.}
\label{app:cifar-schedules}
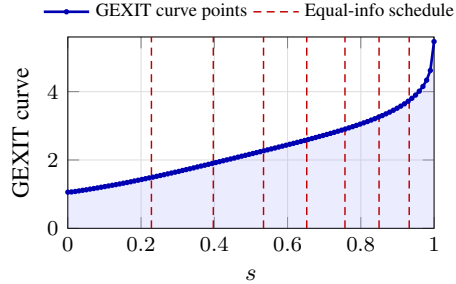
\begin{wrapfigure}{r}{0.46\textwidth}
\vspace{-1.2em}
\centering
\input{shared/figures/tikz/cifar10_exit_equal_area_schedule}
\vspace{-0.6em}
\caption{Equal-information schedule for CIFAR-10 masking.}
\vspace{-0.2em}
\label{fig:app_cifar_equal_area_schedule}
\end{wrapfigure}
For CIFAR-10 sampling, we evaluate three schedule families: linear, cosine, and
equal-information schedules, following the D3PM convention of comparing simple
parametric schedules with an information-balanced schedule
\citep{Austin2021D3PM}. The linear schedule spaces the sampling points uniformly
in the native noise parameter, while the cosine schedule first applies a cosine
reparameterization and then maps the points back to the same endpoint range.

For the equal-information schedule, we first evaluate the conservation-law curve
on the validation set and compute its cumulative area along the noise path. The
sampling points are then chosen so that consecutive intervals contain equal
amounts of area, which makes each reverse step account for approximately the
same contribution to the likelihood-scale integral. As illustrated in
Figure~\ref{fig:app_cifar_equal_area_schedule}, this produces smaller steps in
regions where the curve is large and larger steps where the curve is small.
\FloatBarrier
\vspace{0.8em}

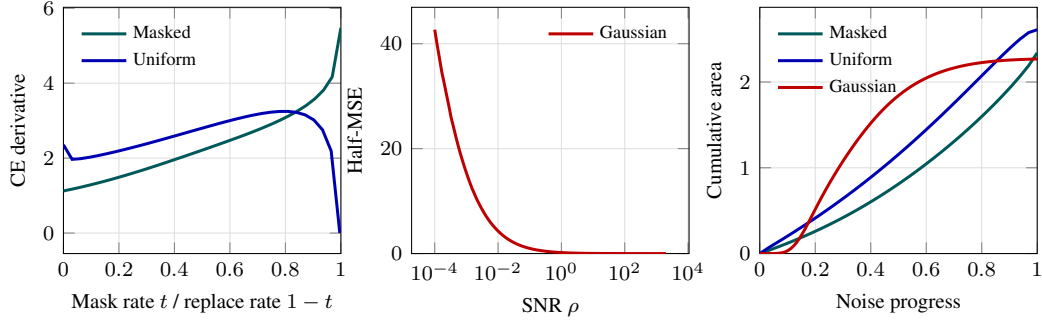
\begin{figure}[!htbp]
\centering
\vspace{2.4em}
\resizebox{\textwidth}{!}{\input{shared/figures/tikz/exp_cifar_conservation_curves}}
\caption{CIFAR-10 conservation-law evaluation curves for the image models. The
left panel shows the discrete-channel derivatives on their native axes, with
masked noise plotted against mask rate $t$ and uniform noise plotted against replacement
rate $1-t$. The middle panel shows the Gaussian curve as half-MSE versus
SNR $\rho$. The right panel compares the cumulative conservation-law area
along each model's native noising path, with zero corresponding to the clean
endpoint and one to the noisiest endpoint.}
\label{fig:app_cifar_conservation_curves}
\end{figure}
\FloatBarrier

\paragraph{Generation and sample quality.}
All CIFAR-10 sample grids and generation metrics use 1000 diffusion steps. For masked noise, generation uses an absorbing-state reverse
chain: the initial state is fully masked, the mask rate is decreased along the
chosen schedule, and the model repeatedly samples clean RGB-token proposals;
tokens that have already been revealed remain fixed, while still-masked tokens
are updated by the conditional masking transition, following the absorbing
discrete-diffusion construction of \citet{Austin2021D3PM}. For uniform noise, we use a
D3PM-style categorical reverse-posterior sampler for the $K$-ary symmetric
channel \citep{Austin2021D3PM}; each reverse transition combines the model
posterior over the clean image with the forward channel kernels so that the
marginal noising schedule is preserved. For Gaussian noise, generation is
performed in the learned continuous RGB-token embedding space with a log-SNR
schedule and a DDPM/score-SDE-style denoising chain
\citep{Ho2020DDPM,Song2021SDE}; the Gaussian-SDE row uses the stochastic
Gaussian bridge sampler with uniform spacing in log SNR. FID and IS are computed
from 50,000 generated images converted to unsigned 8-bit RGB tensors. FID uses
the 2048-dimensional Inception-v3 feature statistic against the corresponding
real CIFAR-10 split, and IS uses the Inception-v3 class-posterior statistic with
10 splits \citep{Heusel2017FID,Salimans2016ImprovedGAN}. When the real split
contains fewer than 50,000 images, it is repeated only to match the generated
image count for metric accumulation.

\paragraph{Resources.}
Each CIFAR-10 simulation used eight NVIDIA RTX 5090 GPUs. Dataset preparation is
CPU-only and is performed once before training.

%% file: shared/figures/tikz/exp_markov_channel_comparison_curves.tex
\begin{tikzpicture}
\begin{groupplot}[
group style={group size=3 by 1, horizontal sep=1.75cm},
width=0.30\linewidth,
height=0.29\linewidth,
grid=both,
major grid style={draw=gray!25},
minor grid style={draw=gray!15},
tick label style={font=\footnotesize},
label style={font=\footnotesize},
title style={font=\footnotesize},
legend style={draw=none, fill=white, fill opacity=0.85, text opacity=1, font=\scriptsize},
every axis plot/.append style={mark size=0.8pt},
]
\nextgroupplot[
xlabel={Mask rate $t$},
ylabel={GEXIT integrand},
xmin=0,
xmax=1,
title={Masked},
legend cell align=left,
legend pos=north west,
]
\addplot+[line width=1.15pt, color=teal!70!black]
table [x=native_x, y=model_integrand, col sep=comma] {shared/figures/data/markov_remote/channel_comparison/exit_overlay.csv};
\addlegendentry{Model}

\addplot+[line width=1.15pt, dashed, color=black!70]
table [x=native_x, y=reference_integrand, col sep=comma] {shared/figures/data/markov_remote/channel_comparison/exit_overlay.csv};
\addlegendentry{BCJR}

\nextgroupplot[
xlabel={Replace rate $1-t$},
ylabel={GEXIT integrand},
xmin=0,
xmax=1,
title={Uniform},
legend cell align=left,
legend pos=south east,
]
\addplot+[line width=1.15pt, color=blue!70!black]
table [x expr=1-\thisrow{native_x}, y=model_integrand, col sep=comma] {shared/figures/data/markov_remote/channel_comparison/qsc_overlay.csv};
\addlegendentry{Model}

\addplot+[line width=1.15pt, dashed, color=black!70]
table [x expr=1-\thisrow{native_x}, y=reference_integrand, col sep=comma] {shared/figures/data/markov_remote/channel_comparison/qsc_overlay.csv};
\addlegendentry{BCJR}

\nextgroupplot[
xlabel={SNR $\rho$},
ylabel={Half-MMSE},
title={Gaussian},
legend cell align=left,
legend pos=north east,
]
\addplot+[line width=1.15pt, color=red!75!black]
table [x=snr, y=model_half_mmse, col sep=comma] {shared/figures/data/markov_remote/channel_comparison/immse_k128.csv};
\addlegendentry{Model}

\addplot+[line width=1.15pt, dashed, color=black!70]
table [x=snr, y=true_half_mmse, col sep=comma] {shared/figures/data/markov_remote/channel_comparison/immse_k128.csv};
\addlegendentry{BCJR}
\end{groupplot}
\end{tikzpicture}

%% file: shared/figures/tikz/cifar10_exit_equal_area_schedule.tex
\begin{tikzpicture}
\begin{axis}[
width=\linewidth,
height=0.64\linewidth,
xlabel={$s$},
ylabel={GEXIT curve},
xmin=0,
xmax=1,
ymin=0,
ymax=5.6,
grid=both,
major grid style={draw=gray!25},
minor grid style={draw=gray!15},
legend columns=2,
legend style={draw=none, fill=none, font=\scriptsize, at={(0.5,1.03)}, anchor=south},
legend cell align={left},
tick label style={font=\footnotesize},
label style={font=\small},
]
\path[name path=baseline] (axis cs:0,0) -- (axis cs:1,0);

\addplot+[name path=modelcurve, draw=none, mark=none, forget plot]
table [x=s, y=exit_gap_nats, col sep=comma] {shared/figures/data/cifar_suite/cifar10_rgb8_exit_bs2048_0850000/eval_5k_1000_ema_equal_info/ema/conservation_curve.csv};
\addplot[draw=none, fill=blue!30, fill opacity=0.25, forget plot]
fill between[of=modelcurve and baseline];

\addplot+[line width=1.0pt, mark size=0.5pt, color=blue!70!black]
table [x=s, y=exit_gap_nats, col sep=comma] {shared/figures/data/cifar_suite/cifar10_rgb8_exit_bs2048_0850000/eval_5k_1000_ema_equal_info/ema/conservation_curve.csv};
\addlegendentry{GEXIT curve points}

\tikzset{equalinfostep/.style={red!75!black, line width=0.55pt, densely dashed}}
\draw[equalinfostep] (axis cs:0.2283,0) -- (axis cs:0.2283,5.6);
\draw[equalinfostep] (axis cs:0.3973,0) -- (axis cs:0.3973,5.6);
\draw[equalinfostep] (axis cs:0.5344,0) -- (axis cs:0.5344,5.6);
\draw[equalinfostep] (axis cs:0.6524,0) -- (axis cs:0.6524,5.6);
\draw[equalinfostep] (axis cs:0.7568,0) -- (axis cs:0.7568,5.6);
\draw[equalinfostep] (axis cs:0.8499,0) -- (axis cs:0.8499,5.6);
\draw[equalinfostep] (axis cs:0.9322,0) -- (axis cs:0.9322,5.6);
\addlegendimage{red!75!black, line width=0.55pt, densely dashed}
\addlegendentry{Equal-info schedule}
\end{axis}
\end{tikzpicture}

%% file: shared/figures/tikz/exp_cifar_conservation_curves.tex
\begin{tikzpicture}
\begin{groupplot}[
group style={group size=3 by 1, horizontal sep=0.95cm},
width=0.38\linewidth,
height=0.35\linewidth,
grid=both,
major grid style={draw=gray!25},
minor grid style={draw=gray!15},
tick label style={font=\footnotesize},
label style={font=\footnotesize},
title style={font=\footnotesize},
legend style={draw=none, fill=none, font=\scriptsize},
]
\nextgroupplot[
xlabel={Mask rate $t$ / replace rate $1-t$},
ylabel={CE derivative},
xmin=0,
xmax=1,
legend cell align=left,
legend pos=north west,
]
\addplot+[line width=1.15pt, color=teal!70!black, mark=none]
coordinates {
(0,1.1238105)
(0.03125,1.172622)
(0.0625,1.224135)
(0.09375,1.2796377)
(0.125,1.3379941)
(0.15625,1.3986476)
(0.1875,1.4633669)
(0.21875,1.5299445)
(0.25,1.599326)
(0.28125,1.670437)
(0.3125,1.7443274)
(0.34375,1.818789)
(0.375,1.8960545)
(0.40625,1.974245)
(0.4375,2.0523069)
(0.46875,2.1320998)
(0.5,2.2120246)
(0.53125,2.2928715)
(0.5625,2.3741819)
(0.59375,2.4572677)
(0.625,2.5403714)
(0.65625,2.6265115)
(0.6875,2.7155543)
(0.71875,2.8089629)
(0.75,2.9072668)
(0.78125,3.0130224)
(0.8125,3.1284421)
(0.84375,3.2573895)
(0.875,3.4050117)
(0.90625,3.5831285)
(0.9375,3.8147186)
(0.96875,4.1673486)
(1,5.4767579)
};
\addlegendentry{Masked}

\addplot+[line width=1.15pt, color=blue!70!black, mark=none]
coordinates {
(0.0001,2.3538735)
(0.031224805,1.9667204)
(0.062349609,1.9855342)
(0.093474414,2.0225087)
(0.12459922,2.0666562)
(0.15572402,2.1164073)
(0.18684883,2.1693046)
(0.21797363,2.2253061)
(0.24909844,2.284031)
(0.28022324,2.3436707)
(0.31134805,2.4059122)
(0.34247285,2.4693356)
(0.37359766,2.5340003)
(0.40472246,2.5989691)
(0.43584727,2.6643949)
(0.46697207,2.7301962)
(0.49809687,2.7952272)
(0.52922168,2.8596941)
(0.56034648,2.9234371)
(0.59147129,2.9847626)
(0.62259609,3.0435416)
(0.6537209,3.0988541)
(0.6848457,3.1489528)
(0.71597051,3.1920631)
(0.74709531,3.2252492)
(0.77822012,3.2453688)
(0.80934492,3.2462862)
(0.84046973,3.2200758)
(0.87159453,3.1522225)
(0.90271934,3.0176643)
(0.93384414,2.754527)
(0.96496895,2.1907191)
(0.99609375,0.0006883496)
};
\addlegendentry{Uniform}

\nextgroupplot[
xlabel={SNR $\rho$},
ylabel={Half-MSE},
xmode=log,
ymin=0,
legend cell align=left,
legend pos=north east,
]
\addplot+[line width=1.15pt, color=red!75!black, mark=none]
coordinates {
(0.0001,42.67543)
(0.00011850803,39.887825)
(0.00014044153,37.310934)
(0.00016643449,34.657726)
(0.00019723824,32.577262)
(0.00023374315,30.381829)
(0.0002770044,28.264303)
(0.00032827246,25.963911)
(0.00038902922,24.194014)
(0.00046103086,22.359194)
(0.00054635859,20.623494)
(0.0006474788,19.01859)
(0.00076731437,17.548935)
(0.00090932914,16.125638)
(0.0010776281,14.824147)
(0.0012770758,13.576464)
(0.0015134373,12.385692)
(0.0017935448,11.38283)
(0.0021254946,10.375872)
(0.0025188817,9.4716361)
(0.0029850771,8.6412864)
(0.0035375561,7.8493123)
(0.004192288,7.1307408)
(0.0049681979,6.4705762)
(0.0058877135,5.8560082)
(0.0069774132,5.3075354)
(0.0082687949,4.8072081)
(0.0097991859,4.3436165)
(0.011612822,3.9196732)
(0.013762127,3.5297875)
(0.016309225,3.1826469)
(0.019327742,2.8507166)
(0.022904926,2.5697151)
(0.027144176,2.3064161)
(0.032168028,2.0674542)
(0.038121697,1.8560633)
(0.045177272,1.6588336)
(0.053538694,1.4870796)
(0.063447652,1.3276958)
(0.075190562,1.1876673)
(0.089106854,1.0613207)
(0.10559878,0.94638251)
(0.12514303,0.84213794)
(0.14830454,0.74965722)
(0.17575279,0.6668648)
(0.20828116,0.59265351)
(0.2468299,0.52690002)
(0.29251326,0.4684314)
(0.3466517,0.41445055)
(0.4108101,0.36718929)
(0.48684295,0.32483528)
(0.57694799,0.28653282)
(0.6837297,0.2529926)
(0.81027459,0.22261492)
(0.96024045,0.19600123)
(1.137962,0.17210066)
(1.3485764,0.1506871)
(1.5981713,0.13231238)
(1.8939613,0.11556757)
(2.2444963,0.10068645)
(2.6599083,0.087736225)
(3.1522049,0.076198502)
(3.7356159,0.066283941)
(4.4270048,0.057373714)
(5.2463562,0.0496911)
(6.2173534,0.043031973)
(7.368063,0.037242295)
(8.7317463,0.03214002)
(10.34782,0.027816277)
(12.262998,0.024031655)
(14.532638,0.02073804)
(17.222342,0.01789896)
(20.409859,0.01545952)
(24.187321,0.013357356)
(28.663918,0.011561591)
(33.969044,0.0099983985)
(40.256045,0.0086399161)
(47.706646,0.0074795354)
(56.536206,0.006458761)
(66.999944,0.0055454142)
(79.400314,0.0047280984)
(94.095747,0.0039796732)
(111.51102,0.0032878815)
(132.14951,0.0026442261)
(156.60778,0.002058669)
(185.59279,0.0015532103)
(219.94236,0.0011155692)
(260.64936,0.0007624746)
(308.89042,0.00049185353)
(366.05995,0.00029239067)
(433.81044,0.00016051825)
(514.1002,8.0367717e-05)
(609.25002,3.5811487e-05)
(722.01019,1.3527479e-05)
(855.64006,4.3115738e-06)
(1014.0022,1.1900964e-06)
(1201.674,2.6982385e-07)
(1424.0802,2.5817022e-08)
(1687.6494,5.8055442e-09)
(2000,2.4155513e-13)
};
\addlegendentry{Gaussian}

\nextgroupplot[
xlabel={Noise progress},
ylabel={Cumulative area},
xmin=0,
xmax=1,
ymin=0,
legend cell align=left,
legend pos=north west,
]
\addplot+[line width=1.15pt, color=teal!70!black, mark=none]
coordinates {
(0,0)
(0.03125,0.035881757)
(0.0625,0.073331086)
(0.09375,0.11245253)
(0.125,0.15335303)
(0.15625,0.19611306)
(0.1875,0.24083203)
(0.21875,0.28760252)
(0.25,0.33649738)
(0.28125,0.38758742)
(0.3125,0.44094312)
(0.34375,0.49661681)
(0.375,0.55466124)
(0.40625,0.61513467)
(0.4375,0.67804954)
(0.46875,0.7434309)
(0.5,0.81130784)
(0.53125,0.88169684)
(0.5625,0.95461955)
(0.59375,1.0301109)
(0.625,1.1081991)
(0.65625,1.1889316)
(0.6875,1.2724014)
(0.71875,1.358722)
(0.75,1.4480381)
(0.78125,1.5405426)
(0.8125,1.636503)
(0.84375,1.7362816)
(0.875,1.8403816)
(0.90625,1.9495713)
(0.9375,2.0651626)
(0.96875,2.1898824)
(1,2.3405716)
};
\addlegendentry{Masked}

\addplot+[line width=1.15pt, color=blue!70!black, mark=none]
coordinates {
(0.00010039216,0)
(0.031347255,0.06723882)
(0.062594118,0.1287454)
(0.09384098,0.19112017)
(0.12508784,0.2547574)
(0.15633471,0.31985592)
(0.18758157,0.38655189)
(0.21882843,0.45494259)
(0.25007529,0.52511871)
(0.28132216,0.59713687)
(0.31256902,0.67105179)
(0.34381588,0.74692236)
(0.37506275,0.82478628)
(0.40630961,0.90466762)
(0.43755647,0.98657821)
(0.46880333,1.070531)
(0.5000502,1.1565199)
(0.53129706,1.244524)
(0.56254392,1.3345234)
(0.59379078,1.4264692)
(0.62503765,1.5202841)
(0.65628451,1.6158745)
(0.68753137,1.7131054)
(0.71877824,1.8117869)
(0.7500251,1.9116557)
(0.78127196,2.012354)
(0.81251882,2.1133798)
(0.84376569,2.2140119)
(0.87501255,2.3131802)
(0.90625941,2.4091984)
(0.93750627,2.4990276)
(0.96875314,2.5759875)
(1,2.6100911)
};
\addlegendentry{Uniform}

\addplot+[line width=1.15pt, color=red!75!black, mark=none]
coordinates {
(0,0)
(0.01010101,8.3191929e-07)
(0.02020202,4.7853829e-06)
(0.03030303,3.5436677e-05)
(0.04040404,0.00016542669)
(0.050505051,0.00058111571)
(0.060606061,0.0017236123)
(0.070707071,0.0044053559)
(0.080808081,0.0097658677)
(0.090909091,0.019186239)
(0.1010101,0.034186198)
(0.11111111,0.056173381)
(0.12121212,0.085946887)
(0.13131313,0.12365327)
(0.14141414,0.16896089)
(0.15151515,0.22080984)
(0.16161616,0.27785232)
(0.17171717,0.33865028)
(0.18181818,0.40157403)
(0.19191919,0.46524303)
(0.2020202,0.52866343)
(0.21212121,0.59121283)
(0.22222222,0.65251251)
(0.23232323,0.71233951)
(0.24242424,0.77070721)
(0.25252525,0.82768174)
(0.26262626,0.88325047)
(0.27272727,0.93747142)
(0.28282828,0.99043436)
(0.29292929,1.042196)
(0.3030303,1.0928063)
(0.31313131,1.1422668)
(0.32323232,1.1905334)
(0.33333333,1.2376594)
(0.34343434,1.2836737)
(0.35353535,1.3285243)
(0.36363636,1.3722242)
(0.37373737,1.4148131)
(0.38383838,1.4562303)
(0.39393939,1.4964384)
(0.4040404,1.5354405)
(0.41414141,1.5732124)
(0.42424242,1.6097504)
(0.43434343,1.6449581)
(0.44444444,1.6788402)
(0.45454545,1.7114482)
(0.46464646,1.7427433)
(0.47474747,1.7727452)
(0.48484848,1.8014681)
(0.49494949,1.8289314)
(0.50505051,1.8551662)
(0.51515152,1.880172)
(0.52525253,1.9040043)
(0.53535354,1.9266805)
(0.54545455,1.9482035)
(0.55555556,1.9686352)
(0.56565657,1.988026)
(0.57575758,2.0064135)
(0.58585859,2.0238466)
(0.5959596,2.0403614)
(0.60606061,2.0559731)
(0.61616162,2.0707076)
(0.62626263,2.0846198)
(0.63636364,2.0977426)
(0.64646465,2.1101131)
(0.65656566,2.1217674)
(0.66666667,2.1327296)
(0.67676768,2.1430426)
(0.68686869,2.1527182)
(0.6969697,2.1618034)
(0.70707071,2.170335)
(0.71717172,2.1783243)
(0.72727273,2.1858029)
(0.73737374,2.1927918)
(0.74747475,2.1993111)
(0.75757576,2.2053828)
(0.76767677,2.2110396)
(0.77777778,2.2163072)
(0.78787879,2.221203)
(0.7979798,2.2257507)
(0.80808081,2.2299665)
(0.81818182,2.2338647)
(0.82828283,2.2374706)
(0.83838384,2.2407955)
(0.84848485,2.2438591)
(0.85858586,2.2466876)
(0.86868687,2.249289)
(0.87878788,2.2516773)
(0.88888889,2.2538661)
(0.8989899,2.2558683)
(0.90909091,2.2577003)
(0.91919192,2.2593746)
(0.92929293,2.2608974)
(0.93939394,2.2622859)
(0.94949495,2.2635536)
(0.95959596,2.2647021)
(0.96969697,2.2657374)
(0.97979798,2.2666721)
(0.98989899,2.2675183)
(1,2.268282)
};
\addlegendentry{Gaussian}
\end{groupplot}
\end{tikzpicture}

%% file: meta/checklist.tex
\section*{NeurIPS Paper Checklist}

\begin{enumerate}

\item {\bf Claims}
    \item[] Question: Do the main claims made in the abstract and introduction accurately reflect the paper's contributions and scope?
    \item[] Answer: \answerYes{} 
    \item[] Justification: The abstract and introduction match the paper's scope: we derive a (mis)matched GEXIT/area-law likelihood characterization for diffusion-style denoising objectives under masked, uniform, and Gaussian noise, and we empirically validate finite-capacity gaps on synthetic Markov sources, \texttt{text8}, and CIFAR-10 (Sections~1--5).

\item {\bf Limitations}
    \item[] Question: Does the paper discuss the limitations of the work performed by the authors?
    \item[] Answer: \answerYes{} 
    \item[] Justification: Section~\ref{sec:limitations} discusses the main theoretical assumptions, the empirical scope of the experiments, and the computational limitation of the Gaussian conditional-mean computation for large discrete supports.

\item {\bf Theory assumptions and proofs}
    \item[] Question: For each theoretical result, does the paper provide the full set of assumptions and a complete (and correct) proof?
    \item[] Answer: \answerYes{} 
    \item[] Justification: Assumptions for the GEXIT setting (admissible, smooth, degraded, coordinatewise channels) are stated in Section~2, and proofs for the main derivative/area identities (including masked and uniform-replacement specializations) are provided in the appendix (Appendix \ref{app:proofs}).

    \item {\bf Experimental result reproducibility}
    \item[] Question: Does the paper fully disclose all the information needed to reproduce the main experimental results of the paper to the extent that it affects the main claims and/or conclusions of the paper (regardless of whether the code and data are provided or not)?
\item[] Answer: \answerYes{} 
    \item[] Justification: The paper specifies the key model/optimization settings in the main text and provides a consolidated reproducibility appendix with dataset preparation, splits, architectures, optimization hyperparameters, seeds, evaluation grids, samplers, and metric definitions sufficient to rerun training/evaluation and regenerate the reported tables/figures (Appendix~\ref{app:repro}).

\item {\bf Open access to data and code}
    \item[] Question: Does the paper provide open access to the data and code, with sufficient instructions to faithfully reproduce the main experimental results, as described in supplemental material?
    \item[] Answer: \answerYes{} 
    \item[] Justification: The main text and Appendix~\ref{app:repro} provides sufficient details to reproduce the experiments, including architectures, training procedures, and hyperparameters. Code will be released upon acceptance. The submission is accompanied by an anonymized supplementary ZIP file containing the implementation.

\item {\bf Experimental setting/details}
    \item[] Question: Does the paper specify all the training and test details (e.g., data splits, hyperparameters, how they were chosen, type of optimizer) necessary to understand the results?
    \item[] Answer: \answerYes{} 
    \item[] Justification: Dataset descriptions/splits and the key architectural and optimization hyperparameters are given in the experiments section, with full details in the main text and appendix (Section~\ref{sec:implementation} and Appendix~\ref{app:repro}).

\item {\bf Experiment statistical significance}
    \item[] Question: Does the paper report error bars suitably and correctly defined or other appropriate information about the statistical significance of the experiments?
    \item[] Answer: \answerYes{} 
    \item[] Justification: The reported large-scale neural experiments use one training seed per objective. The error bars reported for sample-quality metrics therefore quantify fixed-checkpoint evaluation variability, not variability across random initializations or full retraining runs.

    \item[] Guidelines:
    \begin{itemize}
        \item The answer \answerNA{} means that the paper does not include experiments.
        \item The authors should answer \answerYes{} if the results are accompanied by error bars, confidence intervals, or statistical significance tests, at least for the experiments that support the main claims of the paper.
        \item The factors of variability that the error bars are capturing should be clearly stated (for example, train/test split, initialization, random drawing of some parameter, or overall run with given experimental conditions).
        \item The method for calculating the error bars should be explained (closed form formula, call to a library function, bootstrap, etc.)
        \item The assumptions made should be given (e.g., Normally distributed errors).
        \item It should be clear whether the error bar is the standard deviation or the standard error of the mean.
        \item It is OK to report 1-sigma error bars, but one should state it. The authors should preferably report a 2-sigma error bar than state that they have a 96\% confidence interval (CI), if the hypothesis of Normality of errors is not verified.
        \item For asymmetric distributions, the authors should be careful not to show in tables or figures symmetric error bars that would yield results that are out of range (e.g., negative error rates).
        \item If error bars are reported in tables or plots, the authors should explain in the text how they were calculated and reference the corresponding figures or tables in the text.
    \end{itemize}

\item {\bf Experiments compute resources}
    \item[] Question: For each experiment, does the paper provide sufficient information on the computer resources (type of compute workers, memory, time of execution) needed to reproduce the experiments?
    \item[] Answer: \answerYes{} 
    \item[] Justification: The reproducibility appendix reports the GPU type used for the \texttt{text8} and CIFAR-10 simulations and states that dataset preparation is CPU-only; the synthetic experiments are lightweight and described by their training configuration (Appendix~\ref{app:repro}).
    
\item {\bf Code of ethics}
    \item[] Question: Does the research conducted in the paper conform, in every respect, with the NeurIPS Code of Ethics \url{https://neurips.cc/public/EthicsGuidelines}?
    \item[] Answer: \answerYes{} 
    \item[] Justification: The work uses standard public benchmarks/synthetic data, does not involve human subjects or sensitive personal data collection, and is presented as methodological research with clear empirical scope.

\item {\bf Broader impacts}
    \item[] Question: Does the paper discuss both potential positive societal impacts and negative societal impacts of the work performed?
     \item[] Answer: \answerNA{} 
    \item[] Justification: This work is a theoretical and empirical analysis of likelihood-aligned diffusion objectives and does not introduce a deployed system, new dataset, or released high-capability generative model. The experiments use standard public benchmarks and synthetic data. Any downstream societal risks are therefore indirect and are those generally associated with improved generative modeling rather than specific to the proposed conservation-law analysis.
    \item[] Guidelines:
    \begin{itemize}
        \item The answer \answerNA{} means that there is no societal impact of the work performed.
        \item If the authors answer \answerNA{} or \answerNo, they should explain why their work has no societal impact or why the paper does not address societal impact.
        \item Examples of negative societal impacts include potential malicious or unintended uses (e.g., disinformation, generating fake profiles, surveillance), fairness considerations (e.g., deployment of technologies that could make decisions that unfairly impact specific groups), privacy considerations, and security considerations.
        \item The conference expects that many papers will be foundational research and not tied to particular applications, let alone deployments. However, if there is a direct path to any negative applications, the authors should point it out. For example, it is legitimate to point out that an improvement in the quality of generative models could be used to generate Deepfakes for disinformation. On the other hand, it is not needed to point out that a generic algorithm for optimizing neural networks could enable people to train models that generate Deepfakes faster.
        \item The authors should consider possible harms that could arise when the technology is being used as intended and functioning correctly, harms that could arise when the technology is being used as intended but gives incorrect results, and harms following from (intentional or unintentional) misuse of the technology.
        \item If there are negative societal impacts, the authors could also discuss possible mitigation strategies (e.g., gated release of models, providing defenses in addition to attacks, mechanisms for monitoring misuse, mechanisms to monitor how a system learns from feedback over time, improving the efficiency and accessibility of machine learning (ML)).
    \end{itemize}
    
\item {\bf Safeguards}
    \item[] Question: Does the paper describe safeguards that have been put in place for responsible release of data or models that have a high risk for misuse (e.g., pre-trained language models, image generators, or scraped datasets)?
    \item[] Answer: \answerNA{} 
    \item[] Justification: The paper does not introduce or release high-misuse-risk models or scraped datasets; experiments are on synthetic Markov data and standard benchmarks (\texttt{text8}, CIFAR-10).

\item {\bf Licenses for existing assets}
    \item[] Question: Are the creators or original owners of assets (e.g., code, data, models), used in the paper, properly credited and are the license and terms of use explicitly mentioned and properly respected?
    \item[] Answer: \answerYes{}
    \item[] Justification: The paper uses standard public benchmarks, \texttt{text8} and CIFAR-10, and cites their original sources. Published baseline results are attributed to the corresponding papers. The submitted code artifact is our own implementation and is released under the MIT license; dataset download/preprocessing scripts obtain the benchmarks from their public sources and do not redistribute the raw datasets. 

\item {\bf New assets}
    \item[] Question: Are new assets introduced in the paper well documented and is the documentation provided alongside the assets?
    \item[] Answer: \answerNA{} 
    \item[] Justification: The paper does not currently release new datasets or trained models as part of the submission.

\item {\bf Crowdsourcing and research with human subjects}
    \item[] Question: For crowdsourcing experiments and research with human subjects, does the paper include the full text of instructions given to participants and screenshots, if applicable, as well as details about compensation (if any)? 
    \item[] Answer: \answerNA{} 
    \item[] Justification: The work does not involve crowdsourcing or research with human subjects.

\item {\bf Institutional review board (IRB) approvals or equivalent for research with human subjects}
    \item[] Question: Does the paper describe potential risks incurred by study participants, whether such risks were disclosed to the subjects, and whether Institutional Review Board (IRB) approvals (or an equivalent approval/review based on the requirements of your country or institution) were obtained?
    \item[] Answer: \answerNA{} 
    \item[] Justification: No human-subject studies are performed.

\item {\bf Declaration of large language model (LLM) usage}
    \item[] Question: Does the paper describe the usage of LLMs if it is an important, original, or non-standard component of the core methods in this research? Note that if the LLM is used only for writing, editing, or formatting purposes and does \emph{not} impact the core methodology, scientific rigor, or originality of the research, declaration is not required.
    \item[] Answer: \answerNA{} 
    \item[] Justification: LLMs are not an important or non-standard component of the proposed methods or experiments.

\end{enumerate}